\documentclass[twoside,11pt]{article}

\usepackage{jmlr2e}

\usepackage{amsmath}
\usepackage{algorithm, algorithmic}
\usepackage[tight]{subfigure}
\usepackage{color}
\usepackage{wrapfig}

\newcommand{\Rbb}{\mathbb{R}}

\newcommand{\bs}{\boldsymbol}
\DeclareMathOperator{\Tr}{Tr}
\DeclareMathOperator{\diag}{\mathbf{diag}}
\DeclareMathOperator*{\argmin}{arg\,min}
\DeclareMathOperator*{\argmax}{arg\,max}

\ShortHeadings{Learning to Match via Inverse Optimal Transport}{Li, Ye, Zhou and Zha}
\firstpageno{1}

\begin{document}

\title{Learning to Match via Inverse Optimal Transport}

\author{\name Ruilin Li \email ruilin.li@gatech.edu \\
       \addr  School of Mathematics\\
       School of Computational Science and Engineering\\
       Georgia Institute of Technology\\
       Atlanta, GA 30332, USA
       \AND
       \name Xiaojing Ye \email xye@gsu.edu \\
       \addr Department of Mathematics and Statistics\\
       Georgia State University\\
       Atlanta, GA 30302, USA
       \AND
       \name Haomin Zhou \email hmzhou@math.gatech.edu \\
       \addr School of Mathematics\\
       Georgia Institute of Technology\\
       Atlanta, GA 30332, USA
       \AND 
       \name Hongyuan Zha \email zha@cc.gatech.edu \\
       \addr School of Computational Science and Engineering\\
       Georgia Institute of Technology\\
       Atlanta, GA 30332, USA}

\editor{}

\maketitle

\begin{abstract}   
We propose a unified data-driven framework based on inverse optimal transport that can learn adaptive, nonlinear interaction cost function from noisy and incomplete empirical matching matrix and predict new matching in various matching contexts. 
We emphasize that the discrete optimal transport plays the role of a variational principle which gives rise to an optimization based framework for modeling the observed empirical matching data. 
Our formulation leads to a non-convex optimization problem which can be solved efficiently by an alternating optimization method. A key novel aspect of our formulation is the incorporation of marginal relaxation via regularized Wasserstein distance, significantly improving the robustness of the method in the face of noisy or missing empirical matching data.
Our model falls into the category of prescriptive models, which not only predict potential future matching, but is also able to explain what leads to empirical matching and quantifies the impact of changes in matching factors. The proposed approach has wide applicability including predicting matching in online dating, labor market, college application and crowdsourcing. We back up our claims with numerical experiments on both synthetic data and real world data sets.
\end{abstract}


\section{Introduction} \label{introduction}

Matching is a key problem at the heart of many real-world applications, including online dating \citep{hitsch2010matching}, labor market \citep{david2001wiring}, crowdsourcing \citep{yuen2011task}, marriage \citep{becker1973theory}, paper-to-reviewer assignment \citep{charlin2011framework}, kidney transplant donor matching \citep{dickerson2015futurematch} 
and ad allocation \citep{mehta2013online}. Owing to the wide applicability and great importance of matching, 2012 Nobel prize in economics were awarded to two economists Lloyd Shapley and Alvin Roth for their fundamental theoretic work  \citep{gale1962college} and substantive empirical investigations, experiments and practical design \citep{roth1989college, roth1992two} on matching. A good matching of individuals from two sides (e.g., men vs.~women, students vs.~school, papers vs.~reviewers) is essential to the overall health of the specific market/community. However, matching is a challenging problem due to two major complications: individuals from both sides exhibit various observable and latent features, which makes ``suitability of a match'' far more complex to assess; and the matching is implicitly, but significantly, influenced by the supply limitations of individuals from each side, so that the supply of an item can only satisfy a small number of users even though it is preferred by many. These two issues must be properly tackled in an optimal matching system.

In many matching problems, feature and preference data can be collected from individuals of either or both sides of the matching. Then a central planner may use such data sets to infer suitable matching or assignment. The feature and preference data collected in this way, however, can be incomplete, noisy, and biased for two reasons:
\begin{itemize} 
\item an individual may not be aware of the competitors from her own side and/or limited quantity of her preferred match from the opposite side
\item collection of a full spectrum of features is inherently difficult or even infeasible (e.g., a student's merit outside of her school curriculum in college admission, or religious belief of a person in a marriage, may not be included in the collected data)
\end{itemize} 
The former factor prevents individuals from listing their orders of preferences and positioning themselves strategically in the market, and the latter results in feature data set that is incomplete and biased. 

One possible approach is to use observed (and perhaps latent) features of individuals to generate rating matrix for user-item combinations as in many recomender systems (RS). However, this approach is not suitable given the biaseness and noise in collected feature or preference data and limited supply constraints in our matching problems. For example, in a standard movie RS problem, a movie can receive numerous high ratings and be watched by many people. In contrary, in a matching-based college admission problem, a student can enter only one college. Therefore, an optimal matching cannot be obtained solely based on personal ratings and preferences---the population of both sides also need to be taken into consideration in a matching problem. This significant difference between standard recommendation and matching demands for new theoretical and algorithmic developments.

Our approach to tackle the aforementioned challenges in matching inference is to consider a generalized framework based on inverse optimal transport, where the diversified population of each side of the matching is naturally modeled as a probability distribution, and the bilateral preference of two individuals in a potential match is captured by a matching reward (or equivalently, negative matching cost). More specifically, we obtain kernel representation of the cost by learning the feature interaction matrix from the matching data, under which the total social surplus is supposed to be maximal in a healthy matching market as suggested by economists \citep{carlier2010matching}.
Moreover, we employ a robust and flexible Wasserstein metric to learn feature-enriched marginal distributions, which proves to be very effective and robust in dealing with incomplete/noisy data in the matching problem. 

From a broader perspective, our approach is in the framework of optimization based on variational principles---the observed data are results of some optimization with an unknown objective function (or a known objective function with unknown parameters) that models the problem, and the goal is to learn the objective function (or its parameters) from the data. This approach is a type of prescriptive analytics: it exploits the motivation and mechanism of the subject, and produces results that are interpretable and meaningful to human. The solution process is more instructive and can make use of the observed data more effectively. In this broader sense, our proposed approach based on inverse optimal transport is in a similar spirit as inverse reinforcement learning \citep{ng2000algorithms}. Furthermore, the learned objective can be used to understand the effect of various factors in a matching and infer optimal matching strategy given new data. For instance, in online dating, riders allocation, and many other settings, the central planners (Tinder, OkCupid, Uber, Lyft, etc.) can use such prescriptive models to improve customer experience and align social good with their own profit goal. 

Our work is the first to establish a systematic framework for optimal matching inference using incomplete, noisy data under \emph{limited-supply constraints}. In particular, we advocate a nonlinear representation of cost/reward in a matching and view the matching strategy as a solution of (regularized) optimal transport. The equilibrium of certain matching markets, such as marriage, with simplifying assumptions, coincide with optimal transport plans \citep{becker1973theory}. Even for matching markets with complex structure and factors, whose matching mechanism is not yet completely unveiled, the proposed model serves as a powerful modeling tool to study those matchings. In terms of algorithmic development, we derive a highly efficient learning method to estimate the parameters in the cost function representation in the presence of computationally complex Wasserstein metrics. Numerical results show that our method contrasts favorably to other matching approaches in terms of robustness and efficiency, and can be used to infer optimal matching for new data sets accurately.

The rest of this paper is organized as follows: we briefly summarize related work in Section \ref{related} and review discrete optimal transport and its regularized version as well as their close connections in Section \ref{background}. Section \ref{model} describes the setup of proposed model and introduces our robust formulation via regularized Wasserstein distance, which tries to capture matching mechanism by leveraging regularized optimal transport. The derivation of optimization algorithm is detailed in Section \ref{derivation}. We evaluate our model in section \ref{experiments} on both synthetic data and real-world data sets. The last section concludes the paper and points to several directions for potential future research.
\section{Related Work}\label{related}
In this section, we briefly summarize some related work, including matching, ecological inference, recommender systems, distance metric learning and reciprocal recommendation.
\subsection{Matching}
Matching has been widely studied in economics community since the seminal work of \citet{koopmans1957assignment}. \citet{gale1962college} studied optimal matching in college admission, marriage market and proposed the famous Gale-Shapley algorithm. \citet{becker1973theory} gave a theoretic analysis in marriage market matching. \citet{roth1992two} did a thorough study and analysis in two-sided matching. \citet{chiappori2010hedonic, carlier2010matching} used optimal transport theory to study the equilibrium of certain matching markets. \citet{galichon2010matching} theoretically justified the usage of entropy-regularized optimal transport plan to model empirical matching in the presence of unobserved characteristics. Another interesting work \citep{charlin2011framework} proposed to predict optimal matching from learning suitability score in paper-to-review context where they used well-known linear regression, collaborative filtering algorithms to learn suitability scores. 
There are also some work studying dynamic matching theory and applications such as kidney exchange \citep{dickerson2012dynamic, dickerson2015futurematch} and barter exchange \citep{anderson2017efficient, ashlagi2017matching}.

A recent work closely related to ours is \citep{dupuy2016estimating}, where they worked with regularized optimal transport plan and modeled the cost by a bilinear form using an affinity matrix learned from data. 
By contrast, our work models the cost using a nonlinear kernel representation and incorporate regularized Wasserstein distance to tackle the challenging issues due the incomplete and noisy data in real-world matching problems. 

\subsection{Ecological Inference}
Ecological inference infers the nature of individual level behavior using aggregate (historically called ``ecological'') data, and is of particular interest to political scientists, sociologists, historians and epidemiologists. Due to privacy or cost issue, individual level data are eluding from researchers, hence the inference made through aggregate data are often subject to ecological fallacy \footnote{\url{https://en.wikipedia.org/wiki/Ecological\_fallacy}}. Previously, people proposed neighborhood model \citep{freedman1991ecological}, ecological regression \citep{goodman1953ecological} and King's method \citep{king2013solution}. A recent progress \citep{flaxman2015supported} is made by using additional information and leverage kernel embeddings of distributions, distribution regression to approach this problem. 

Our work differs from classical ecological inference problem and methods in four ways. First, we assume access to empirical matching at individual-level granularity which is not available in standard ecological inference setting. Second, in out framework, we focus on learning the preference of two sides in the matching and propose a novel and efficient method to learn it, after which inference/prediction problem becomes trivial as preference is known. Third, different from previous statistical methods, we adopt a model-based approach, leverages optimal transport to model matching and draw a connection between these two fields. Lastly, thanks to the model-based approach, we are able to shed light on what factors lead to empirical matching and quantitatively estimate the influence caused by changes of those factors, which are beyond the reach of traditional statistical approaches.

\subsection{Recommender Systems}
Collaborative filtering (CF) type recommender systems share many similarities with optimal matching problem as both need to learn user preference from rating/matching data and predict rating/matching in a collaborative manner. Matrix-factorization based models \citep{mnih2008probabilistic, salakhutdinov2008bayesian} enjoyed great success in Netflix Prize Competition. \citet{rendle2010factorization, rendle2012factorization} proposed factorization machine model  with strong sparse predictive power and ability to mimic several state-of-the-art, specific factorization methods. Recently there is trend of combining collaborative filtering with deep learning \citep{he2017neural2, he2017neural1}. Most items recommended by conventional recommender systems, however, are non-exclusive and can be consumed by many customers such as movies and music. They do not take supply limit of either or both sides into consideration hence may perform poorly in matching context.

\subsection{Distance Metric Learning} 
Our model essentially aims to learn an adaptive, nonlinear representation of the matching cost.
This is closely related to, but more general than, ground metric learning.
Prior research on learning different distance metrics in various contexts are fruitful, such as learning cosine similarity for face verification \citep{nguyen2010cosine}, learning Mahalanobis distance for clustering \citep{xing2003distance} and face identification \citep{guillaumin2009you}. However, distance learning for optimal transport distance is largely unexplored. \citet{cuturi2014ground} proposed to learn the ground metric by minimizing the difference of two convex polyhedral functions.  \citet{wang2012supervised} formulated a SVM-like minimization problem to learn Earth Mover's distance. Both approaches work with Wasserstein distance which involves solving linear programming as subroutine hence may be computationally too expensive. This paper works with regularized optimal transport distance, involving solving a matrix scaling problem as subroutine which is much lighter than linear programming. 

\subsection{Reciprocal Recommendation} Another line of related research is reciprocal recommendation \citep{brozovsky2007recommender, pizzato2013recommending}, which also tries to model two-side preference by computing reciprocal score via a hand-craft score function. By a sharp contrast, our model learns how two sides interact with each other from observed noisy/incomplete matching in a data-driven fashion.

\section{Background and Preliminaries} \label{background}
In this section, we present Kantorovich's formulation of optimal transportation problem (in discretized setting) and its regularized version.

\subsection{Optimal Transport}
Given two probability vectors $\bs{\mu} \in \Sigma_m$ and $\bs{\nu} \in \Sigma_n$, where 
$\Sigma_d := \{ \bs{x} \in \mathbb{R}_+^{d}| \bs{1}^T\bs{x} =1\}$ 
is the standard $(d-1)$-dimensional probability simplex, denote the transport polytope of $\bs{\mu}$ and $\bs{\nu}$ by 
\[ U(\bs{\mu}, \bs{\nu}) := \{ \pi \in \mathbb{R}^{m \times n}_+ | \pi \bs{1} = \bs{\mu}, \pi^T\bs{1} = \bs{\nu} \}\]
namely the set of all $m\times n$ non-negative matrices satisfying marginal constraints specified by $\bs{\mu}, \bs{\nu}$. Note that $U(\bs{\mu}, \bs{\nu})$ is a convex, closed and bounded set containing joint probability distributions with $\bs{\mu}$ and $\bs{\nu}$ as marginals.
Furthermore, if given a cost matrix $C = [C_{ij}]\in \mathbb{R}^{m \times n}$ where $C_{ij}$ measures the cost of moving a unit mass from ${\mu}_i$ to ${\nu}_j$,  define
\[ d(C, \bs{\mu}, \bs{\nu}) := \min_{\pi \in U(\bs{\mu}, \bs{\nu})} \langle \pi, C\rangle \]
where $\langle A, B\rangle = \Tr(A^TB)$ is the Frobenius inner product for matrices. This quantity describes how to optimally redistribute $\bs{\mu}$ to $\bs{\nu}$ so that the total cost is minimized, hence providing a means to measure the similarity between the two distributions. In particular, when $C \in \mathcal{M}^{d}$, that is, $C$ is in the cone of distance matrices \citep{brickell2008metric}, defined as 
\[ \mathcal{M}^{d} := \{C \in \mathbb{R}_+^{d \times d}| C_{ii} = 0, C_{ij}=C_{ji},  C_{ij}\le C_{ik}+C_{kj}, \forall i, j, k\},\] 
then it is shown that $d(C)$ is a distance (or metric) on $\Sigma_d$ \citep{villani2008optimal}, named the optimal transport distance (also known as the 1-Wasserstein distance or the earth mover distance). The minimizer $\pi$ is called the optimal transport plan. 

In discrete case, computing OT distance amounts to solving a linear programming problem, for which there exists dedicated algorithm with time complexity $\mathcal{O}(n^3\log n)$ \citep{pele2009fast}. Nevertheless, this is still too computationally expensive in large scale settings. In addition, OT plan $\pi$ typically admits a sparse form which is not robust in data-driven applications.
We refer readers to \cite{villani2008optimal, peyre2017computational} for a thorough theoretical and computational treatment of optimal transport.

\subsection{Regularized Optimal Transport}
To address the aforementioned computational difficulty, \citet{cuturi2013sinkhorn} proposed to use a computationally-friendly approximation of OT distance by introducing entropic regularization. This also mitigates the sparsity and improve the smoothness of OT plan.
Concretely, consider
\[
d_\lambda(C, \bs{\mu}, \bs{\nu}) := \min_{\pi \in U(\bs{\mu}, \bs{\nu})} \{\langle \pi, C\rangle - H(\pi) / \lambda \}
\]
where $H(\pi)$ is the discrete entropy defined by
\[ 
H(\pi) = -\sum_{i, j=1}^{m,n} \pi_{ij}(\log \pi_{ij} -1),
\] 
and $\lambda>0$ is the regularization parameter controlling the trade-off between sparsity and uniformity of $\pi$. We refer the above quantity as regularized optimal transport (ROT) distance (regularized Wasserstein distance) though it is \textit{not} an actual distance measure. Due to the strict convexity introduced by entropy, $d_\lambda(C, \bs{\mu}, \bs{\nu})$ admits a unique minimizer with full support $\pi^\lambda(C, \bs{\mu}, \bs{\nu})$, which we  call regularized optimal transport plan in the sequel. The ROT plan $\pi^\lambda$ has a semi-closed form solution 
\begin{equation}\label{eq:semi-closed}
\pi^\lambda = \diag(\bs{a}) K \diag(\bs{b})
\end{equation}
where $\bs{a} \in \mathbb{R}^{m}, \bs{b} \in \mathbb{R}^{n}$ are positive vectors and are uniquely determined up to a multiplicative constant and $K := \exp(-\lambda C)$ is the component-wise exponential of $-\lambda C$. We can efficiently compute $\bs{a}$ and $\bs{b}$ by Sinkhorn-Knopp matrix scaling algorithm \citep{sinkhorn1967concerning}, also known as iterative proportional fitting procedure (IPFP). 
The algorithm alternately scales rows and columns of $K$ to fit the specified marginals. See Algorithm \ref{alg:sinkhorn} for detailed description of the Sinkhorn-Knopp algorithm. 

\begin{algorithm}[tb]
   \caption{Sinkhorn-Knopp Algorithm}
   \label{alg:sinkhorn}
\begin{algorithmic}
   \STATE {\bfseries Input:} marginal distributions $\bs{\mu}, \bs{\nu}$, cost matrix $C$, regularization parameter $\lambda$
   \STATE $K = \exp(-\lambda C)$ 
   \STATE $\bs{a} = \bs{1}$
   \WHILE {not converge}
   \STATE $\bs{b} \gets \frac{\bs{\nu}}{K^T\bs{a}}$
   \STATE $\bs{a} \gets \frac{\bs{\mu}}{K\bs{b}}$
   \ENDWHILE
   \STATE $\pi = \diag(\bs{a}) K \diag(\bs{b})$
   \STATE \textbf{return} $\pi, \bs{a}, \bs{b}$
\end{algorithmic}
\end{algorithm}

Not surprisingly, we have 
\[
\lim_{\lambda \to \infty} d_\lambda(C, \bs{\mu}, \bs{\nu}) = d(C, \bs{\mu}, \bs{\nu})\]
ROT distance converges to OT distance as $\lambda$ tends to infinity, i.e., entropic regularization diminishes. Moreover, let 
\[
\Pi(C, \bs{\mu}, \bs{\nu}) = \{ \pi| \langle \pi, C \rangle = \min_{\pi \in U(\bs{\mu}, \bs{\nu})} \langle \pi, C\rangle \}
\]
be the set of all OT plan and 
\[
\pi^\star = \argmax_{\pi \in \Pi(C, \bs{\mu}, \bs{\nu})} H(\pi)
\]
be the joint distribution with highest entropy within $\Pi(C, \bs{\mu}, \bs{\nu})$, then 
\[ 
\lim_{\lambda \to \infty} \pi^\lambda = \pi^\star
\]
in another word, ROT plan converges to the most uniform OT plan and the rate of convergence is exponential, as shown by \citet{cominetti1994asymptotic}. 
The generalization of entropic regularization, Tsallis entropy regularized optimal transport also receives more and more attention and is studied by \citet{muzellec2017tsallis}.

ROT has more favorable computational properties than OT does, as it only involves component-wise operation and matrix-vector multiplication, all of which are of quadratic complexity, and can be parallelized \citep{cuturi2013sinkhorn}. This fact makes ROT popular for measuring dissimilarity between 
potentially unnormalized distributions in many research fields: machine learning \citep{geneway2017learning, rolet2016fast, laclau2017co}, computer vision \citep{cuturi2014fast} and image processing \citep{papadakis2015optimal}. 

Besides computation efficiency, we argue in next section why it is more appropriate to use ROT in our setting from a modeling perspective.
\section{Learning to Match} \label{model}
For the ease of exposition, we refer two sides of the matching market as users and items. The methodology is suitable in various applications where optimal matching is considered under supply limitations, such as marriage market, cab hailing, college admission,  organ allocation, paper matching ans so on.
Suppose we have $m$ user profiles $\{\bs{u}_i\}_{i \in [m]} \subset \mathbb{R}^{p}$, $n$ item profiles $\{\bs{v}_j\}_{j \in [n]} \subset \mathbb{R}^{q}$ and $N_{ij}$, the count of times $(\bs{u}_i, \bs{v}_j)$ appears in matching. Let $N = \sum_{i, j=1}^{m, n} N_{ij}$ be the number of all matchings, $[\hat{\pi}_{ij}] = [N_{ij} / N]$ be the observed matching matrix and $\hat{\bs{\mu}} = \hat{\pi}\bs{1}, \hat{\bs{\nu}} = \hat{\pi}^T\bs{1}$ be the sample marginals. Suppose we are also given two cost matrices $C_u$ and $C_v$, measuring user-user dissimilarity and item-item dissimilarity respectively, we can then select two appropriate constants $\lambda_u$ and $\lambda_v$ and use $d_{\lambda_u}(C_u, \bs{\mu}_1, \bs{\mu}_2)$ and $d_{\lambda_v}(C_v, \bs{\nu}_1, \bs{\nu}_2)$ to measure the dissimilarity of probability distributions $\bs{\mu}_1, \bs{\mu}_2$ over user profile space and that of $\bs{\nu}_1, \bs{\nu}_2$ over item profile space.

\subsection{Modeling Observed Matching Matrix}
\cite{becker1973theory} pointed out that equilibrium of some matching markets coincide with optimal transport plans which are often highly sparse. The implication of this theory is far from being realistic, though, as we observe heterogeneous matchings in real world. \cite{galichon2015cupid} argued that there are latent features having significant impact on matching but unfortunately unobservable to researchers. Hence they proposed to leverage a combination of pure optimal transport plan and mutual information of two sides of matching to model empirical matching data which is exactly entropy-regularized optimal transport.

Furthermore, the observed matching matrix $\hat{\pi}$ (hence the empirical marginals) often contains noisy, corrupted, and/or missing entries, consequently it is more robust to employ a regularized optimal transport plan rather than enforce an exact matching to empirical data in cost function learning.

To that end, we propose to use regularized optimal transport plan $\pi^\lambda(C, \bs{\mu}, \bs{\nu})$ in our learning task. This also has several important benefits that take the following aspects into modeling consideration in addition to unobserved latent features:

\begin{itemize}

\item \textbf{Enforced Diversity.} Diversity is enforced in certain matchings as is the case when admission committee making decisions on applicants, diversity is often an important criterion and underrepresented minorities may be preferred. Entropy term captures the uncertainty introduced by diversity. The idea of connecting entropy with matching to capture/promote diversity is also adopted, for example, by \citet{agrawal2018proportional} and \citet{ahmed2017diverse}.

\item \textbf{Aggregated Data.} Sometimes due to privacy issues or insufficient number of matched pairs, only grouped or aggregated data, rather than individual data are available. Accordingly, the aggregated matching is usually denser than individual level matching and is less likely to exhibit sparsity.
\end{itemize}

\subsection{Cost Function via Kernel Representation}
The cost function $C_{ij} = c(\bs{u}_i, \bs{v}_j)$ is of critical importance as it determines utility loss of user $\bs{u}_i$ and item $\bs{v}_j$. The lower the cost is, the more likely user $\bs{u}_i$ will match item $\bs{v}_j$, subject to supply limit of items. A main contribution of this work is to learn an adaptive, nonlinear representation of the cost function from empirical matching data. To that end, we present several properties of cost function in optimal matching that support the feasibility.

First of all, we show in the following proposition that the cost function $C$ is not unique in general but can be uniquely determined in a special and important case.
\begin{proposition}\label{prop:1}
Given two marginal probability vectors $\bs{\mu}\in \Sigma_{m}$, $\bs{\nu}\in \Sigma_{n}$, define $F: \mathbb{R}^{m \times n} \to U(\bs{\mu}, \bs{\nu})$, $F(C) = \pi^\lambda(C, \bs{\mu}, \bs{\nu})$ is the ROT plan of $C$.
Then $F$ is in general not injective, however, when $m = n$ and $F$ is restricted on $\mathcal{M}^n$, $F_{|\mathcal{M}^n}(C)$ is injective.

\end{proposition}
%
\begin{proof}
One can easily verify that $F$ is well-defined from the strict convexity of ROT. The optimality condition of ROT reads as 
\[ \pi^\lambda(C, \bs{\mu}, \bs{\nu}) = \exp(\lambda(-C + \bs{a}\bs{1}^T + \bs{1}\bs{b}^T))\]
where $\bs{a} \in \Rbb^m$ and  $\bs{b}\in\Rbb^n$ are Lagrangian multipliers dependent on $C$ and $\lambda$ such that $\pi^\lambda(C, \bs{\mu}, \bs{\nu}) \in U(\bs{\mu},\bs{\nu})$.
Therefore, $\pi^\lambda(C+\epsilon \bs{1}\bs{1}^T, \bs{\mu}, \bs{\nu}) = \exp(\lambda(-C -\epsilon \bs{1}\bs{1}^T + (\bs{a}+\epsilon \bs{1})\bs{1}^T + \bs{1}\bs{b}^T)) = \pi^\lambda(C, \bs{\mu}, \bs{\nu})$ for any $\epsilon>0$.
Therefore $F$ is in general not injective.

If $m = n$ and $C_1, C_2 \in \mathcal{M}^n$,  by the semi-closed form \eqref{eq:semi-closed} of ROT plan, there exist positive vectors $\bs{a}_1, \bs{b}_1$ and $\bs{a}_2, \bs{b}_2$ such that 
    \begin{align*}
        \pi^\lambda(C_1, \bs{\mu}, \bs{\nu}) &= \diag(\bs{a}_1) \exp({-\lambda C_1}) \diag(\bs{b}_1) \\
        \pi^\lambda(C_2, \bs{\mu}, \bs{\nu}) &= \diag(\bs{a}_2) \exp({-\lambda C_2}) \diag(\bs{b}_2)
    \end{align*}
    If $\pi^\lambda(C_1, \bs{\mu}, \bs{\nu}) = \pi^\lambda(C_2, \bs{\mu}, \bs{\nu})$, we have
    \begin{equation*}
        \exp({-\lambda C_1}) = \diag(\bs{a}) \exp({-\lambda C_2}) \diag(\bs{b})
    \end{equation*}
    where $\exp(\cdot)$ is component-wise exponential, $\bs{a} = \log \frac{\bs{a}_2}{\bs{a}_1}$, $\bs{b} = \log\frac{\bs{b}_2}{\bs{b}_1}$.

    Since $C_1, C_2$ are symmetric matrices, it follows that $\bs{a} = s\bs{b}$. By appropriately rescaling $\bs{a}$ and $\bs{b}$ to make them equal, we have 
    \begin{equation*}
        \exp(-\lambda C_1) = \diag(\bs{w}) \exp(-\lambda C_2) \diag(\bs{w})
    \end{equation*}
    where $\bs{w} = \bs{a} / \sqrt{s}$.
    Inspecting $(i,i)$ entry of both sides, we immediately conclude that $\bs{w} = \bs{1}$ and $C_1 = C_2$.
\end{proof}
Actually, the general non-uniqueness or non-identifiability of cost $C$ is quite natural. For instance, in an online auction setting, if all bidders raise their bids by the same amount, the result of the auction will not change because the rank of bidders remain the same and the original winner still wins the auction. Therefore, by observing empirical matching alone, we can not determine cost matrix definitively without further assumption. Proposition \ref{prop:1} guarantees the uniqueness of learned cost if we model it as a distance matrix, e.g. Mahalanobis distance ($C_{ij} =\sqrt{ (\bs{u}_i - \bs{v}_j)^T M (\bs{u}_i - \bs{v}_j)}$, where $M$ is a positive definite matrix). However, in many cases, cost may grow nonlinearly in the difference of features. An even more serious issue is that if the number of features of two sides of matching are inconsistent or two sides do not lie in the same feature space at all, it would be infeasible to use a distance metric to capture the cost between them due to such dimension incompatibility.

Therefore, as generalized distance functions \citep{scholkopf2001kernel}, 
kernel representation which is able to measure matching cost even when features of two sides do not lie in the same feature space can be leveraged to model the cost function, i.e.,

\[ c(\bs{u}_i, \bs{v}_j) = k(G\bs{u}_i, D\bs{v}_j)\]
where $k(\bs{x}, \bs{y})$ is a specific (possibly nonlinear) kernel, $G \in \mathbb{R}^{r \times p}$ and $D \in \mathbb{R}^{r \times q}$ are two unknown linear transformations to be learned. $G\bs{u},  D\bs{v}$ can be interpreted as the latent profile associated with users and items and are studied by \citet{agarwal2009regression}.

For a wide class of commonly used kernels including linear kernel, polynomial kernel and sigmoid kernel, they depend only on the inner product of two arguments through an activation function $f$, i.e. $ k(\bs{x}, \bs{y}) = f(\bs{x}^T\bs{y})$. For such kernels, we have 
\[ 
c(\bs{u}_i, \bs{v}_j) = f(\bs{u}_i^T G^TD \bs{v}_j)
\] 
and it suffices to learn $A = G^T D$. In this case, cost matrix 
\[ 
C(A)= f(U^TAV)
\]
is parametrized by $A$ and we refer $A$ as interaction matrix. Here we apply $f$ component-wise on $U^TAV$. For ease of presentation, we will work with kernels of this form in the sequel. With kernel function representation, it is still likely that a matching matrix corresponds to multiple cost matrices, and we will be contented with finding one of them that explains the observed empirical matching.


\subsection{Kernel Inference with Wasserstein Marginal Regularization}
A straight forward way to learn $C(A)$ in kernel representation is estimating parameter $A$ through minimizing negative log likelihood
\begin{equation}\label{eq:fix}
\min_{A} -\sum_{i=1}^m \sum_{j=1}^n \hat{\pi}_{ij}\log\pi_{ij}
\end{equation}
where $\pi = \pi^\lambda(C(A), \hat{\bs{\mu}}, \hat{\bs{\nu}})$, i.e., one enforces the optimal plan $\pi$ to satisfy $\pi \bs{1} = \hat{\bs{\mu}}$ and $\pi^T \bs{1} = \hat{\bs{\nu}}$. Note that \eqref{eq:fix} is equivalent to minimizing the reverse Kullback-Leibler divergence \citep{Bishop_2006} of ROT plan $\pi$ with respect to empirical matching $\hat{\pi}$, i.e.,
\[
    \min_{A} \mbox{KL}(\hat{\pi} \| \pi)
\] 
This is the formulation proposed in  \cite{dupuy2016estimating} which we refer as inverse optimal transport formulation (IOT) in the sequel. 

In this variation principle based framework, the ROT plan $\pi$ has the same marginals as the empirical matching $\hat{\pi}$ does, which is reasonable if the marginal information of empirical matching is sufficiently accurate. In practice, however, the size of samples available is usually small compared to that of population, hence the empirical marginals inferred from samples can be incomplete and noisy, which causes a \emph{systematic error} no smaller than 
$O(\max\{\|\Delta \bs{\mu}\|_1,\|\Delta\bs{\nu}\|_1\})$ as shown in proposition \ref{prop:2}.

\begin{lemma}\label{lemma:1}
Supppose $\bs{\mu}_1, \bs{\mu}_2 \in \Sigma_m$, $\bs{\nu}_1, \bs{\nu}_2 \in \Sigma_n$ and $\Delta \bs{\mu} = \bs{\mu}_1 -\bs{\mu}_2, \Delta\bs{\nu} = \bs{\nu}_1-\bs{\nu}_2$, we have
\begin{equation*} 
        \min_{\substack{\pi_1 \in U(\bs{\mu}_1, \bs{\nu}_1)\\ \pi_2 \in U(\bs{\mu}_2, \bs{\nu}_2)}} \|\pi_1 - \pi_2\|_F^2 \ge \frac{m\|\Delta \bs{\mu}\|_2^2 + n\|\Delta \bs{\nu}\|^2_2}{mn} 
\end{equation*}
where $\displaystyle \|\pi\|_F = \sqrt{\sum_{i=1}^m\sum_{j=1}^n\pi_{ij}^2}$ is Frobenius norm.
\end{lemma}
\begin{proof}
See Appendix \ref{app:1}.
\end{proof}

\begin{proposition}\label{prop:2}
If empirical $\hat{\bs{\mu}}, \hat{\bs{\nu}}$ are off from true $\bs{\mu}, \bs{\nu}$ by $\Delta \bs{\mu}, \Delta \bs{\nu}$, then the matching matrix $\pi_{\text{IOT}}$ recovered by solving equation \eqref{eq:fix} has error lower bounded by
\begin{equation*} 
   \|\pi_0 - \pi_{\text{IOT}}\|_1 \ge \sqrt{\frac{\|\Delta \bs{\mu}\|_1^2 + \|\Delta \bs{\nu}\|_1^2}{mn}} 
\end{equation*}
where $\|\pi\|_1 = \sum_{i,j=1}^{m,n} |\pi_{ij}|$, $\hat{\bs{\mu}},\bs{\mu} \in \mathbb{R}^m$ and $\hat{\bs{\nu}}, \bs{\nu} \in \mathbb{R}^n$, $\pi_0$ is the ground truth matching matrix, $\pi_{\text{IOT}} = \pi^\lambda(C(A^\star), \hat{\bs{\mu}}, \hat{\bs{\nu}})$ and $A^\star$ is the solution of equation \eqref{eq:fix}.
\end{proposition}
\begin{proof}
We know $\pi_0 \in U(\bs{\mu}, \bs{\nu})$ and $\pi_{\text{fix}} \in U(\hat{\bs{\mu}}, \hat{\bs{\nu}})$, by lemma \ref{lemma:1} and inequalities between 1-norm and 2-norm of vectors
\[ \|\bs{x}\|_2 \le \|\bs{x}\|_1 \le \sqrt{d}\|\bs{x}\|_2\]
where $\bs{x} \in \mathbb{R}^d$,
we have 
\begin{align*} 
\|\pi_0 - \pi_{\text{fix}}\|_1 &\ge \min_{\substack{\pi_1 \in U(\bs{\mu}, \bs{\nu})\\ \pi_2 \in U(\hat{\bs{\mu}}, \hat{\bs{\nu}})}} \|\pi_1 - \pi_2\|_1 \\
&\ge (\min_{\substack{\pi_1 \in U(\bs{\mu}, \bs{\nu})\\ \pi_2 \in U(\hat{\bs{\mu}}, \hat{\bs{\nu}})}} \|\pi_1 - \pi_2\|_F^2)^\frac{1}{2} \\
&\ge  \sqrt{\frac{m\| \Delta\bs{\mu} \|^2_2 + n \|\Delta\bs{\nu}\|_2^2}{mn}} \\
&\ge \sqrt{\frac{\| \Delta\bs{\mu} \|^2_1 + \|\Delta\bs{\nu}\|^2_1}{mn}}
\end{align*}
\end{proof}

We have seen that inaccurate marginal information can serious harm the recovery performance of ground truth matching matrix. Not unexpectedly, it could mislead us to learn an inaccurate cost matrix as well,  as stated in proposition \ref{prop:2.5}.

\begin{lemma}\label{lemma:2}
Suppose $M \in \mathbb{R}^{m \times n}$ and $f(\bs{a}, \bs{b}) = \|\bs{a}\bs{1}^T + \bs{1}\bs{b}^T - M\|^2_F$.
Then we have 
\[ 
f(\bs{a}, \bs{b}) \ge \|M\|_F^2 - \bs{f}^T A^+ \bs{f}
\]
where $\bs{f} = [(M \bs{1})^T, \bs{1}^TM]^T$, $A = \begin{bmatrix} 
n I_{m\times m} & \bs{1}_m \bs{1}_n^T \\
\bs{1}_n \bs{1}_m^T & mI_{n \times n}
\end{bmatrix}$, $A^+$ is the Moore-Penrose inverse of matrix $A$ and $\|M\|_F = \sqrt{\sum_{i=1}^m \sum_{j=1}^n M_{ij}^2}$ is Frobenius norm. In particular, if $M$ can not be written as $M = \bs{a}\bs{1}^T + \bs{1}\bs{b}^T$, the lower bound is strictly positive, i.e.,
\[ f(\bs{a}, \bs{b}) \ge \|M\|_F^2 - \bs{f}^T A^+ \bs{f} > 0
\]
\end{lemma}
\begin{proof}
See Appendix \ref{app:2}.
\end{proof}

\begin{proposition}\label{prop:2.5}
Suppose $\pi_0 \in \mathbb{R}^{m \times n}$ is the ground truth matching matrix, $\hat{\pi} \in \mathbb{R}^{m \times n}$ is an empirical matching matrix. Let $C_0$ be the ground truth cost matrix giving rise to $\pi_0$ and $C_{IOT}$ be the learned cost matrix via IOT formulation that gives rise to $\hat{\pi}$, i.e. $\pi_0 = \pi^\lambda(C_0, \pi_0\bs{1}, \pi_0^T\bs{1})$ and $\hat{\pi} = \pi^\lambda(C_{IOT}, \hat{\pi}\bs{1}, \hat{\pi}^T\bs{1})$. Denote $\Delta C = C_0 - C_{IOT}$and $\Delta \log \pi = \log \pi_0 - \log \hat{\pi}$ and further assume $(\Delta \log \pi)_{ij}$ are independent (absolutely) continuous random variables (w.r.t. Lebesgue meaure), we have 
\begin{equation}
\|\Delta C\|^2_F \ge \frac{1}{\lambda^2}(\|\Delta \log \pi\|^2_F - \bs{f}^T A^+ \bs{f} ) > 0 \qquad \mbox{a.e.}
\end{equation}
where $\bs{f} = [(\Delta \log \pi\bs{1})^T, \bs{1}^T \Delta \log \pi]^T$, $A = \begin{bmatrix} 
n I_{m\times m} & \bs{1}_m \bs{1}_n^T \\
\bs{1}_n \bs{1}_m^T & mI_{n \times n}
\end{bmatrix}$, $A^+$ is the Moore-Penrose inverse of matrix $A$ and $\|M\|_F = \sqrt{\sum_{i=1}^m \sum_{j=1}^n M_{ij}^2}$ is Frobenius norm.
\end{proposition}
\begin{proof}
By the optimality condition of ROT,  we know that there exist $\bs{a}, \bs{b}$ such that 
\[
\pi^\lambda(C, \bs{\mu}, \bs{\nu}) = \exp(\lambda(-C + \bs{a}^T\bs{1} + \bs{1}\bs{b}^T))
\]
hence there exist $\bs{a}_0, \bs{b}_0, \hat{\bs{a}}, \hat{\bs{b}}$ such that 
\begin{align*}
C_0 &= \bs{a}_0\bs{1}^T + \bs{1}\bs{b}_0^T - \frac{1}{\lambda}\log\pi_0 \\
C_{IOT} &= \hat{\bs{a}}\bs{1}^T + \bs{1}\hat{\bs{b}}^T - \frac{1}{\lambda}\log\hat{\pi}
\end{align*}
Take difference and denote $\bs{a}_0 - \hat{\bs{a}}, \bs{b}_0 - \hat{\bs{b}}$ by $\Delta \bs{a}, \Delta \bs{b}$ respectively, we have 
\[
\Delta C = \Delta \bs{a} \bs{1}^T + \bs{1} \Delta \bs{b}^T - \frac{1}{\lambda} \Delta \log \pi
\]
Since singular matrices have zero Lebesgue measure and $(\Delta \log \pi)_{ij}$ are independent continuous random variables, we have
\[
\mathbb{P}(\Delta \log \pi = \bs{a}\bs{1}^T + \bs{1}\bs{b}^T) \le \mathbb{P}(\mbox{det}(\Delta \log \pi) = 0) = 0
\]
By lemma \ref{lemma:2}, we obtain
\[
\|\Delta C\|^2_F \ge \frac{1}{\lambda^2}(\|\Delta \log \pi\|^2_F - \bs{f}^T A^+ \bs{f} ) > 0 \qquad a.e.
\]
\end{proof}

If we use the inaccurate cost matrix learned via IOT approach, it could negatively affect the quality of future matching prediction,  as justified in proposition \ref{prop:predict_pi}. 

\begin{proposition}\label{prop:predict_pi}
Let $C_0$ be any ground truth cost matrix, $C_{IOT}$ be any learned cost matrix via IOT formulation and assume $C_{IOT} \not\in \{C|C = C_0 + \bs{a}\bs{1}^T + \bs{1}\bs{b}^T \mbox{ for some } \bs{a}, \bs{b} \}$. Suppose the ground truth matching matrix is $\pi_0 = \pi^\lambda(C_0, \bs{\mu}, \bs{\nu})$ and the predicted matching matrix is $\pi_{predict} = \pi^\lambda(C_{IOT}, \bs{\mu}, \bs{\nu})$. Denote $\Delta C = C_0 - C_{IOT}$ and $\Delta \log \pi = \log \pi_0 - \log \pi_{predict}$, we have 
\begin{equation}
\|\Delta \log \pi\|^2_F \ge \lambda^2(\|\Delta C\|^2_F - \bs{f}^T A^+ \bs{f}) > 0
\end{equation}
where $\bs{f} = [(\Delta C \bs{1})^T, \bs{1}^T \Delta C]^T$, $A = \begin{bmatrix} 
n I_{m\times m} & \bs{1}_m \bs{1}_n^T \\
\bs{1}_n \bs{1}_m^T & mI_{n \times n}
\end{bmatrix}$, $A^+$ is the Moore-Penrose inverse of matrix $A$ and $\|M\|_F = \sqrt{\sum_{i=1}^m \sum_{j=1}^n M_{ij}^2}$ is Frobenius norm.
\end{proposition}
\begin{proof}
The proof is almost identical to that of proposition \ref{prop:2.5} except for interchanging the role of $\Delta C$ and $\Delta \log \pi$. We hence omit the details here.
\end{proof}

To address aforementioned issues, we hence propose a robust formulation with Wasserstein marginal relaxation, dropping the hard marginal constraint. Concretely, we consider the following optimization problem.
\begin{equation}\label{eq:primal}
\min_{A, \bs{\mu}\in \Sigma_m, \bs{\nu}\in\Sigma_n} -\sum_{i=1}^m \sum_{j=1}^n \hat{\pi}_{ij}\log\pi_{ij}+  \delta \big(d_{\lambda_u}(C_u, \bs{\mu}, \hat{\bs{\mu}}) + d_{\lambda_v}(C_v, \bs{\nu}, \hat{\bs{\nu}})\big)
\end{equation}
where $\pi = \pi^\lambda(C(A), \bs{\mu}, \bs{\nu})$ is the regularized optimal transport plan, $\delta$ is the relaxation parameter controlling the fitness of marginals, $\lambda, \lambda_u, \lambda_v$ are hyper-parameters controlling the regularity of regularized Wasserstein distance. 
We refer this formulation as robust inverse optimal transport (RIOT) 
in the sequel. Interestingly, we note that \citet{chizat2016scaling} proposed a similar but different formulation in solving unbalanced optimal transport problem.


The intuition of this RIOT formulation is that instead of enforcing noisy empirical marginals as hard constraints, we incorporate them as soft constraints in objective function. We use regularized Wasserstein distance as regularization because of the following reasons: 
\begin{itemize}
\item as approximated Wasserstein distance, it drives $\bs{\mu}, \bs{\nu}$ to $\hat{\bs{\mu}},\hat{\bs{\nu}}$, but at the same time it also allows some uncertainty hence is able to robustify the result;
\item in presence of missing entries in marginals, Wassertein distance is still well defined while other measures such as KL are not;
\item Wasserstein distance can be applied to continuous, discrete, or even mixed distributions;
\item computation of regularized Wasserstein distance \citep{cuturi2013sinkhorn} is efficient and hence potentially more scalable for large scale problem \eqref{eq:primal} in practice.
\end{itemize}


We assume access to $C_u$ and $C_v$ in our model because learning user-user/item-item similarity is relatively easier than our task, there are many existing work dedicated to that end \citep{cheung2004learning, agarwal2013collaborative} and we want to single out and highlight our main contribution---learning the cost matrix that gives rise to observed matching and leverage it to infer matching for new data sets. In fact, our framework can also be extended to learn $C_u$ and $C_v$ jointly if needed, the optimization algorithm of which tends to be much more complex, though. See Appendix \ref{app:extension} for the extension. We postpone the detailed algorithmic derivation of the solution to \eqref{eq:primal} to next section.


\subsection{Predict New Matching}
After obtaining interaction matrix $A$ from solving RIOT, we may then leverage it to predict new matching. Concretely, for a group of new users $\{\tilde{\bs{u}}_i\}_{i \in [m]}$ and items $\{\tilde{\bs{v}}_j\}_{j \in [n]}$, two marginal distributions, i.e., users profile distribution $\tilde{\bs{\mu}}$ and item profile distribution $\tilde{\bs{\nu}}$ can be easily obtained. First compute the cost matrix $\tilde{C}_{ij} = f(\tilde{\bs{u}}_i^T A \tilde{\bs{v}}_j)$ using kernel representation and apply Sinkhorn-Knopp algorithm to computing $\tilde{\pi}^\lambda(\tilde{C}, \tilde{\bs{\mu}}, \tilde{\bs{\nu}})$, which gives us the predicted matching of the given groups of users and items. \\

\noindent See Figure \ref{fig:illustration} for illustration of the complete pipeline or proposed learning-to-match framework.

\begin{figure}[ht]
\begin{center}
\centerline{\includegraphics[width=1.5\columnwidth]{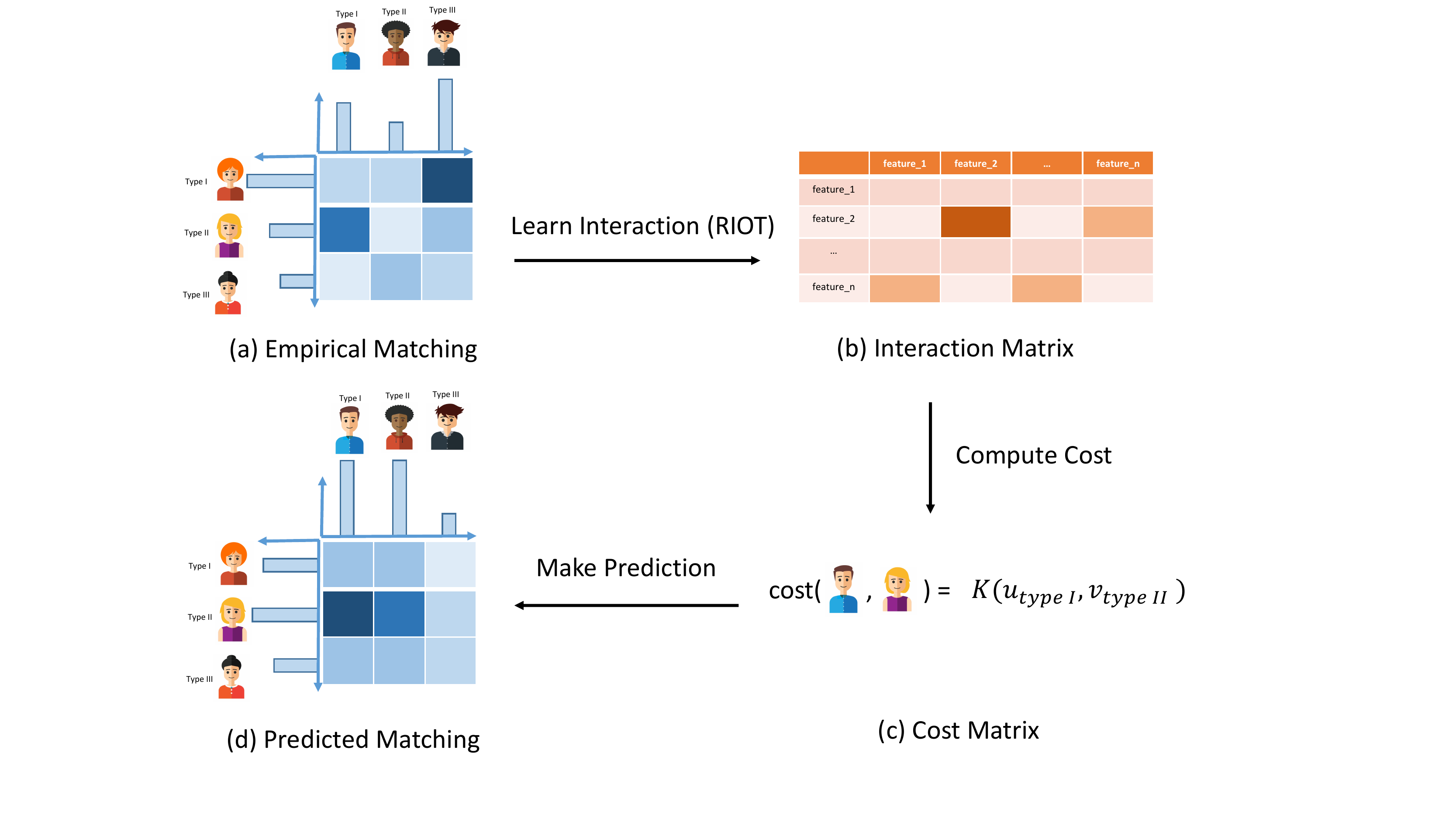}}
\caption{From noisy (a) empirical matching matrix, we learn (b) the interaction matrix via our proposed RIOT formulation. We then use kernel representation to compute (c) cost matrix and predict (d) matching matrix for new data.  $\bs{u}_{type\, I}, \bs{v}_{type\, II}$ are feature vectors of type I men and type II women.}
\label{fig:illustration}
\end{center}
\end{figure}

\section{Derivation of Optimization Algorithm}\label{derivation}

Since the constraint set of ROT problem satisfies Slater's condition \citep{boyd2004convex}, we have by strong duality that
\begin{equation*}
d_\lambda(C, \bs{\mu}, \bs{\nu}) = \max_{\bs{z}} \langle \bs{z}, \bs{\mu}\rangle + \langle \bs{z}^C, \bs{\nu} \rangle - \frac{1}{\lambda}
\end{equation*}
where $z^C_j = \frac{1}{\lambda}\log c_j-\frac{1}{\lambda} \log(\sum_{i=1}^m e^{\lambda(z_i - C_{ij})})$. $\bs{z},\bs{z}^C$ are essentially the Lagrangian multipliers corresponding to constraints $\pi\bs{1}=\bs{\mu}$ and $\pi^T\bs{1}=\bs{\nu}$. See also \cite{genevay2016stochastic}. Hence we have
\begin{equation*}
d_{\lambda_u}(C_u, \bs{\mu}, \hat{\bs{\mu}}) = \max_{\bs{z}} \langle \bs{z}, \bs{\mu}\rangle + \langle \bs{z}^{C_u}, \hat{\bs{\mu}} \rangle - \frac{1}{\lambda_u}
\end{equation*}
\begin{equation*}
d_{\lambda_v}(C_v, \bs{\nu}, \hat{\bs{\nu}}) = \max_{\bs{w}} \langle \bs{w}, \bs{\nu}\rangle + \langle \bs{w}^{C_v}, \hat{\bs{\nu}} \rangle - \frac{1}{\lambda_v}
\end{equation*}
where $z^{C_u}_j = \frac{1}{\lambda_u}\log \hat{r}_j-\frac{1}{\lambda_u} \log(\sum_{i=1}^m e^{\lambda_u(z_i - {C_u}_{ij})})$ and $w^{C_v}_j = \frac{1}{\lambda_v}\log \hat{c}_j-\frac{1}{\lambda_v} \log(\sum_{i=1}^n e^{\lambda_v(w_i - {C_v}_{ij})})$. 
Given sample marginals, once $\bs{z}, \bs{w}$ are fixed, $\bs{z}^{C_u}, \bs{w}^{C_v}$ are also fixed. We can then convert \eqref{eq:primal} into a min-max problem
\begin{equation}\label{eq:dual}
\min_{A,\bs{\mu},\bs{\nu}}\max_{\bs{z},\bs{w}} -\sum_{i=1}^m \sum_{j=1}^n \hat{\pi}_{ij}\log\pi_{ij} + \delta \big( \langle \bs{z},\bs{\mu}\rangle + \langle \bs{z}^{C_u},\hat{\bs{\mu}}\rangle
+\langle \bs{w},\bs{\nu}\rangle + \langle \bs{w}^{C_v},\hat{\bs{\nu}}\rangle \big)
\end{equation}
where constants are omitted. The optimal solution is a saddle-point
of the objective in \eqref{eq:dual}. To solve this min-max problem, we alternately update the primal variable
$(A, \bs{\mu}, \bs{\nu})$ and dual variable $(\bs{z}, \bs{w})$, each time with the other ones fixed. 

\subsection{Update $(A, \bs{\mu}, \bs{\nu})$ for fixed $(\bs{z}, \bs{w})$} \label{subsection:update1}
Now $\bs{z}, \bs{w}, \bs{z}^{C_u},\bs{w}^{C_v}$ are all fixed.
Note that 
\[ 
\pi_{ij} = e^{\lambda(a_i+b_j-C_{ij})}
\]
for some positive vectors $\bs{a},\bs{b}$, such that
$\pi\bs{1} =\bs{\mu}$, $\pi^T\bs{1}=\bs{\nu}$, and $\bs{1}^T\pi\bs{1}=1$.
Thus we may rewrite the minimization in this stage as

\begin{equation}\label{eq:step1}
\begin{array}{ll}
\displaystyle \min_{A,\bs{a},\bs{b}} & \displaystyle \sum_{j=1}^n \hat{\pi}_{ij}\log\pi_{ij}
+\delta(\langle z,\pi \bs{1}\rangle 
+\langle w,\pi^T \bs{1}\rangle)\\
\displaystyle \textrm{s.t.} & \displaystyle \sum_{i=1}^m \sum_{j=1}^ne^{\lambda(a_i+b_j-C_{ij})}=1 \\
\end{array}
\end{equation}

Recall that the ultimate goal in this step is to find the interaction matrix $A$ that cost $C$ depends on, such that the minimum above can be attained. For any $A$, we have kernel representation $C(A)$ parameterized by interaction matrix $A$. Therefore the minimization above is equivalent to
$
\min_A  E(C(A)) 
$.
To minimize $E(C(A))$, the critical step is to evaluate gradient $\nabla_A E(C(A))$ and by envelope theorem \citep{milgrom2002envelope} we have
\begin{proposition}
The gradient $\nabla_A E(C(A))$ is 
\begin{equation*}
\nabla_A E = \sum_{i=1}^m \sum_{j=1}^n \lambda[\hat{\pi}_{ij} + (\theta - \delta(z_i + w_j)) \pi_{ij}] C_{ij}^{\prime}(A)
\end{equation*}
where $\theta$ is the Lagrangian multiplier of the constrained
minimization problem in equation \eqref{eq:step1}.
\end{proposition}

\begin{proof}
By chain rule, we have that 
\begin{equation*}
    \nabla_A E =\sum_{i=1}^m \sum_{j=1}^n \frac{\partial E}{\partial C_{ij}}\frac{\partial C_{ij}}{A}
\end{equation*}
With the kernel representation, $C_{ij}^\prime(A)$ is easily available. For fixed $C=C(A)$, by envelop theorem \citep{milgrom2002envelope}, we have
\begin{equation*}\label{eq:gradE}
\begin{split}
\nabla_{C_{ij}} E(C) &= \frac{\partial}{\partial C} -\sum_{i=1}^m \sum_{j=1}^n\hat{\pi}_{ij}\log\pi_{ij} + \delta\langle \bs{z},\pi \bs{1}\rangle + \delta\langle \bs{w},\pi^T \bs{1}\rangle - \theta(\sum_{i,j=1}^{m,n}e^{\lambda(a_i+b_j-C_{ij})})\\
&= (-\frac{\hat{\pi}_{ij}}{\pi_{ij}} +\delta
(z_i + w_j) - \theta) \frac{\partial \pi_{ij}}{\partial C_{ij}}\\
&=\lambda[\hat{\pi}_{ij} + (\theta - \delta(z_i + w_j))\pi_{ij}]
\end{split}
\end{equation*}
\end{proof}

Hence in each evaluation of $\nabla_C E$, we need to solve $E(C(A))$ once. If we denote $\xi_i=e^{\lambda a_i}$,
$\eta_j=e^{\lambda b_j}$,  $Z_{ij}=e^{-\lambda C_{ij}}$ and $M_{ij}=\delta(z_i + w_j)Z_{ij}$,
then computing $E(C(A))$ is equivalent to solving 

\begin{equation}\label{eq:step2}
\begin{array}{ll}
\displaystyle \min_{\bs{\xi}, \bs{\eta}} & \displaystyle -\langle \hat{\bs{\mu}},\log\bs{\xi} \rangle
-\langle \hat{\bs{\nu}},\log\bs{\eta} \rangle + \bs{\xi}^T M \bs{\eta}\\
\displaystyle \textrm{s.t.} & \displaystyle \bs{\xi}^TZ\bs{\eta}=1 \\
\end{array}
\end{equation}

Note that this is a non-convex optimization problem, both the objective function and constraints are non-convex which is difficult to solve in general. However, once we fix $\bs{\eta}$, the problem with respect to $\bs{\xi}$ alone is a convex problem and vice versa. We can solve this problem efficiently by alternately updating $\bs{\xi},\bs{\eta}$.
\begin{proposition}
Denote the objective in equation \eqref{eq:step2} by 
\[ 
h(\bs{\xi}, \bs{\eta}) = -\langle \hat{\bs{\mu}},\log\bs{\xi} \rangle
-\langle \hat{\bs{\nu}},\log\bs{\eta} \rangle
+ \bs{\xi}^T M \bs{\eta}
\] 
Initialize $\bs{\xi}^{(0)},\bs{\eta}^{(0)}$ and alternately update $\bs{\xi}^{(k)},\bs{\eta}^{(k)}$ in the following fashion
\begin{equation}\label{eq:min1}
\bs{\xi}^{(k)} = \argmin_{\bs{\xi}^T Z\bs{\eta}^{(k-1)}=1} h(\bs{\xi}, \bs{\eta}^{(k-1)})
\end{equation}
\begin{equation}\label{eq:min2}
\bs{\eta}^{(k)} = \argmin_{{\bs{\xi}^{(k)}}^T Z\bs{\eta}=1} h(\bs{\xi}^{(k)}, \bs{\eta})
\end{equation}
If $\lim_{k \to \infty} (\bs{\xi}^{(k)}, \bs{\eta}^{(k)}, \theta_1^{(k)}, \theta_2^{(k)}) = (\bs{\xi}^\star, \bs{\eta}^\star, \theta_1^\star, \theta_2^\star)$, then $\theta_1^\star = \theta_2^\star$ and $(\bs{\xi}^\star, \bs{\eta}^\star)$ is a local minimizer of $h$, where $\theta_1^{(k)}$ and $\theta_2^{(k)}$ are Lagrangian multipliers corresponding to problem \eqref{eq:min1} and \eqref{eq:min2} respectively.
\end{proposition}
\begin{proof}
From the definition of $E(C(A))$ in equation \eqref{eq:step1} it is easily seen that $h(\bs{\xi}, \bs{\eta})$ is lower bounded. Moreover, since 
\[
h(\bs{\xi}^{(k)}, \bs{\eta}^{(k)}) \le h(\bs{\xi}^{(k)}, \bs{\eta}^{(k-1)}) \le h(\bs{\xi}^{(k-1)}, \bs{\eta}^{(k-1)})
\]
there exists a convergent subsequence of $\{h(\bs{\xi}^{(k)}, \bs{\eta}^{(k)})\}$ and we denote the limit by $h^\star$.

The KKT condition of equation \eqref{eq:min1} and \eqref{eq:min2} are
\begin{align}
-\frac{\hat{\bs{\mu}}}{\bs{\xi}^{(k)}} + M\bs{\eta}^{(k-1)} - \theta_1^{(k)} Z\bs{\eta}^{(k-1)} &= 0 \label{eq:kkt1}\\
{\bs{\xi}^{(k)}}^T Z \bs{\eta}^{(k-1)} = 1 \nonumber \\
-\frac{\hat{\bs{\nu}}}{\bs{\eta}^{(k)}} + M^T\bs{\xi}^{(k)} - \theta_2^{(k)} Z^T\bs{\xi}^{(k)} &= 0 \label{eq:kkt3} \\
{\bs{\xi}^{(k)}}^T Z \bs{\eta}^{(k)} = 1 \nonumber 
\end{align}

Let $k$ tend to infinity and take inner product with $\bs{\xi}^\star$ for equation \eqref{eq:kkt1} and take inner product with $\bs{\eta}^\star$ for equation \eqref{eq:kkt3}, compare two equations and use the fact that both $\hat{\bs{\mu}}$ and $\hat{\bs{\nu}}$ are probability vectors, we find that $\theta_1^\star = \theta_2^\star$ and $(\bs{\xi}^\star, \bs{\eta}^\star)$ solves the KKT condition of constrained problem \eqref{eq:min1} and \eqref{eq:min2}. Therefore, $(\bs{\xi}^\star, \bs{\eta}^\star)$ is a local minimizer and $h^\star = h(\bs{\xi}^\star, \bs{\eta}^\star)$ is a local minimum.
\end{proof}

Once we obtain $(\bs{\xi}^\star, \bs{\eta}^\star, \theta^\star)$, we can then get $(\bs{a},\bs{b},\theta)$ by setting $\bs{a}=\frac{1}{\lambda}\log\bs{\xi}^\star$ and $\bs{b}=\frac{1}{\lambda}\log\bs{\eta}^\star$ and $\theta = \theta^\star$. Then plug in $(\bs{a},\bs{b},\theta)$  to evaluate 
$\nabla_C E$ for the current $C=C(A)$.

A careful analysis of the KKT condition of equations \eqref{eq:min1} and \eqref{eq:min2} shows that $\theta_1^{(k)}$ and $\theta_2^{(k)}$ are roots of 
\[
    p(\theta) = \bigg \langle \frac{\hat{\bs{\mu}} \odot (Z\bs{\eta}^{(k-1)})}{ (M -\theta Z)\bs{\eta}^{(k-1)}}, \bs{1} \bigg \rangle, \enskip q(\theta) = \bigg \langle \frac{\hat{\bs{\nu}} \odot (Z^T\bs{\xi}^{(k)})}{ (M -\theta Z)^T\bs{\xi}^{(k)}}, \bs{1} \bigg \rangle 
\]
respectively. The univariate root finding problem can be solved efficiently by off-the-shelf package. After obtaining $\theta_1^{(k)},\theta_2^{(k)}$, we can update 
\[ 
\bs{\xi}^{(k)} = \frac{\hat{\bs{\mu}}}{(M-\theta_1^{(k)}Z)\bs{\eta}^{(k-1)}}, \enskip \bs{\eta}^{(k)} = \frac{\hat{\bs{\nu}}}{(M-\theta_2^{(k)}Z)^T\bs{\xi}^{(k)}}\] 
directly. Computationally, this approach to solving problem \eqref{eq:min1} and \eqref{eq:min2} is much cheaper than gradient-type iterative methods when $m$ and/or $n$ are large.

\subsection{Update $(\bs{z},\bs{w})$ for fixed $(A,\bs{\mu}, \bs{\nu})$} 
When $(A,\bs{\mu}, \bs{\nu})$ are fixed, $\pi$ is also fixed, we then only need to solve 
\[
\max_{\bs{z},\bs{w}}\langle \bs{z},\pi\bs{1}\rangle + \langle \bs{z}^{C_u},\hat{\bs{\mu}}\rangle
+\langle \bs{w},\pi^T\bs{1}\rangle + \langle \bs{w}^{C_v},\hat{\bs{\nu}}\rangle
\]
and one immediately recognizes this is equivalent to applying Sinkhorn-Knopp algorithm to compute
$d_{\lambda_u}(C_u,\bs{\mu}, \hat{\bs{\mu}})$ and $d_{\lambda_v}(C_v,\bs{\nu}, \hat{\bs{\nu}})$.\\

\noindent To summarize, in each iteration, we perform a gradient-type update for $A$, followed by two calls of Sinkhorn-Knopp algorithm to compute $d_{\lambda_u}(C_u,\bs{\mu}, \hat{\bs{\mu}})$ and $d_{\lambda_v}(C_v,\bs{\nu}, \hat{\bs{\nu}})$. Algorithm \ref{alg:main} details the algorithm.

\begin{algorithm}[h] 
   \caption{Solve RIOT}\label{alg:main}
\begin{algorithmic}
   \STATE {\bfseries Input:} observed matching matrix $\hat{\pi}$, cost matrices $C_u, C_v$, regularization parameter $\lambda, \lambda_u, \lambda_v$
   \FOR{$l = 1,2,\cdots, L$}
   \STATE $Z \gets \exp(-\lambda C)$ 
   \STATE $M \gets \delta (\bs{z}\bs{1}^T + \bs{1}\bs{w}^T) \odot Z$
   \STATE Initialize $\bs{\xi}^{(0)}, \bs{\eta}^{(0)}$
        \FOR{$k = 1,2,\cdots, K$}
        \STATE $\theta_1^{(k)} \gets  \mbox{root of } p(\theta)$
        \STATE $\theta_2^{(k)} \gets \mbox{ root of } q(\theta)$
        \STATE $\bs{\xi}^{(k)} \gets \frac{\hat{\bs{\mu}}}{(M-\theta_1^{(k)}Z)\bs{\eta}^{(k-1)}}$
        \STATE $\bs{\eta}^{(k)} \gets \frac{\hat{\bs{\nu}}}{(M-\theta_2^{(k)}Z)^T\bs{\xi}^{(k)}}$
        \ENDFOR
        \STATE $\bs{a} \gets \frac{1}{\lambda} \log\bs{\xi}^{(k)}, \quad \bs{b} \gets \frac{1}{\lambda} \bs{\eta}^{(k)}, \quad \theta = \theta_1^{(k)}$
        \STATE $\pi \gets \exp(\lambda(\bs{a}\bs{1}^T + \bs{1}\bs{b}^T - C))$
        \STATE $\nabla_A  \gets \displaystyle \sum_{i,j=1}^{m,n}\lambda [\hat{\pi}_{ij} + (\theta-\delta(z_i + w_j)\pi_{ij}]C^\prime_{ij}(A)$
        \STATE $A \gets A - s \nabla_A $ 
        \STATE $\bs{a}_1 \gets \mbox{Sinkhorn-Knopp}(C_u, \pi\bs{1}, \hat{\bs{\mu}}, \lambda_u)[1] $
        \STATE $\bs{a}_2 \gets \mbox{Sinkhorn-Knopp}(C_v, \pi^T\bs{1}, \hat{\bs{\nu}}, \lambda_v)[1]$ 
        \STATE $\bs{z} \gets \frac{1}{\lambda_u} \log \bs{a}_1 $, \quad $\bs{w} \gets \frac{1}{\lambda_v} \log \bs{a}_2 $
   \ENDFOR
\end{algorithmic}
\end{algorithm}

\section{Experiments} \label{experiments}
In this section, we evaluate our proposed RIOT model on both synthetic data and real world data sets. For synthetic data set, we illustrate its robustness against IOT and show our model can achieve better performance in learning cost matrix $C$ than IOT could. For election data set, we show our method can effectively learn meaningful preference of voters based on their demographics. For taxi trip data set, we demonstrate that the proposed model is able to predict matching of  taxi drivers and passengers fairly accurate. For marriage data set, we demonstrate the applicability of RIOT in predicting new matching and make recommendation accordingly by comparing it with baseline and state-of-art recommender systems. 

\subsection{Synthetic Data}
We set $\lambda=\lambda_u=\lambda_v=1$ and simulate $m=20$ user profiles $\{\bs{u}_i\} \subset \mathbb{R}^{10}$, $n=20$ item profiles $\{\bs{v}_j\} \subset \mathbb{R}^{8}$, two probability vectors $\bs{\mu}_0, \bs{\nu}_0 \in \mathbb{R}^{20}$, an interaction matrix $A_0$ of size $10\times 8$ and pick polynomial kernel $k(\bs{x}, \bs{y}) = (\gamma\bs{x}^T\bs{y} + c_0)^d$ where $\gamma=0.05, c_0=1, d=2$, hence ${C_0}_{ij} = (0.05\bs{u}_i^TA\bs{v}_j + 1)^2$. For $C_u, C_v$, we randomly generate $m$ and $n$ points from $\mathcal{N}(\bs{0}, 5I_2)$ on plane and use their Euclidean distance matrix as $C_u$ and $C_v$. The ground truth entropy-regularized optimal transport plan is given by $\pi_0 = \pi^\lambda(C_0, \bs{\mu}_0, \bs{\nu}_0) $.
We then add noise to obtain 
\[ \hat{\pi}_{ij} = \frac{{\pi_0}_{ij} + |\epsilon_{ij}|}{\sum_{i,j=1}^{m,n} {\pi_0}_{ij} + |\epsilon_{ij}|}\] 
where $\epsilon_{ij}$ are independent and identical $\mathcal{N}(0, \sigma^2)$ random variables. In algorithm \ref{alg:main}, we set the number of iterations of inner loop $K=20$.

\begin{figure}[ht]
\begin{center}
\centerline{\includegraphics[width=0.75\columnwidth]{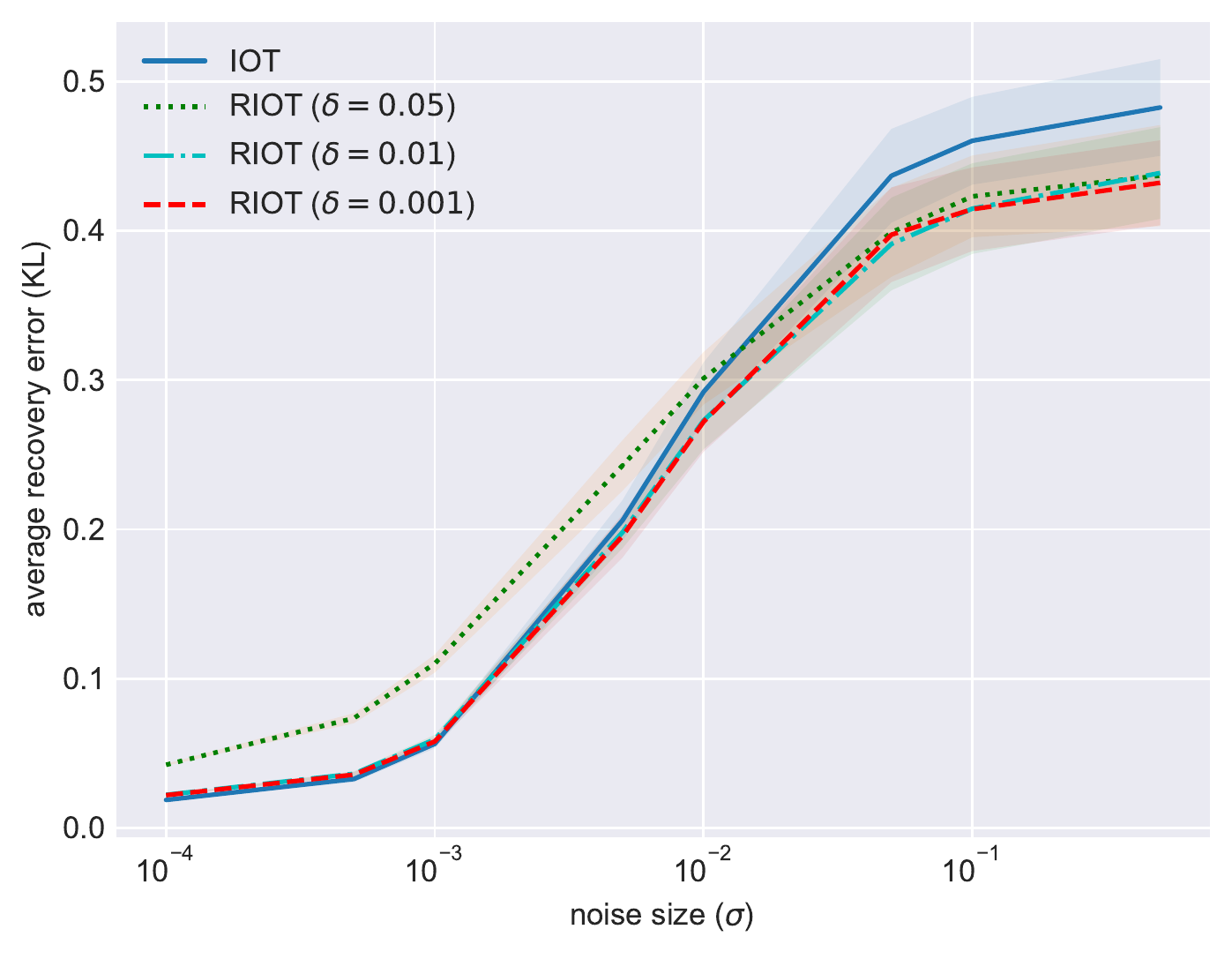}}
\caption{Comparison of recovery performance of fixed marginal approach and marginal relaxation approach. For each noise size $\sigma$, we simulate $\epsilon_{ij}$ and run algorithm \ref{alg:main} 50 times. The shaded region is one standard deviation.}
\label{fig:compare1}
\end{center}
\end{figure}

\begin{figure}[ht]
\centering
\subfigure[$\pi_0$]{%
\label{fig:actual_pi}%
\includegraphics[width=0.25\columnwidth]{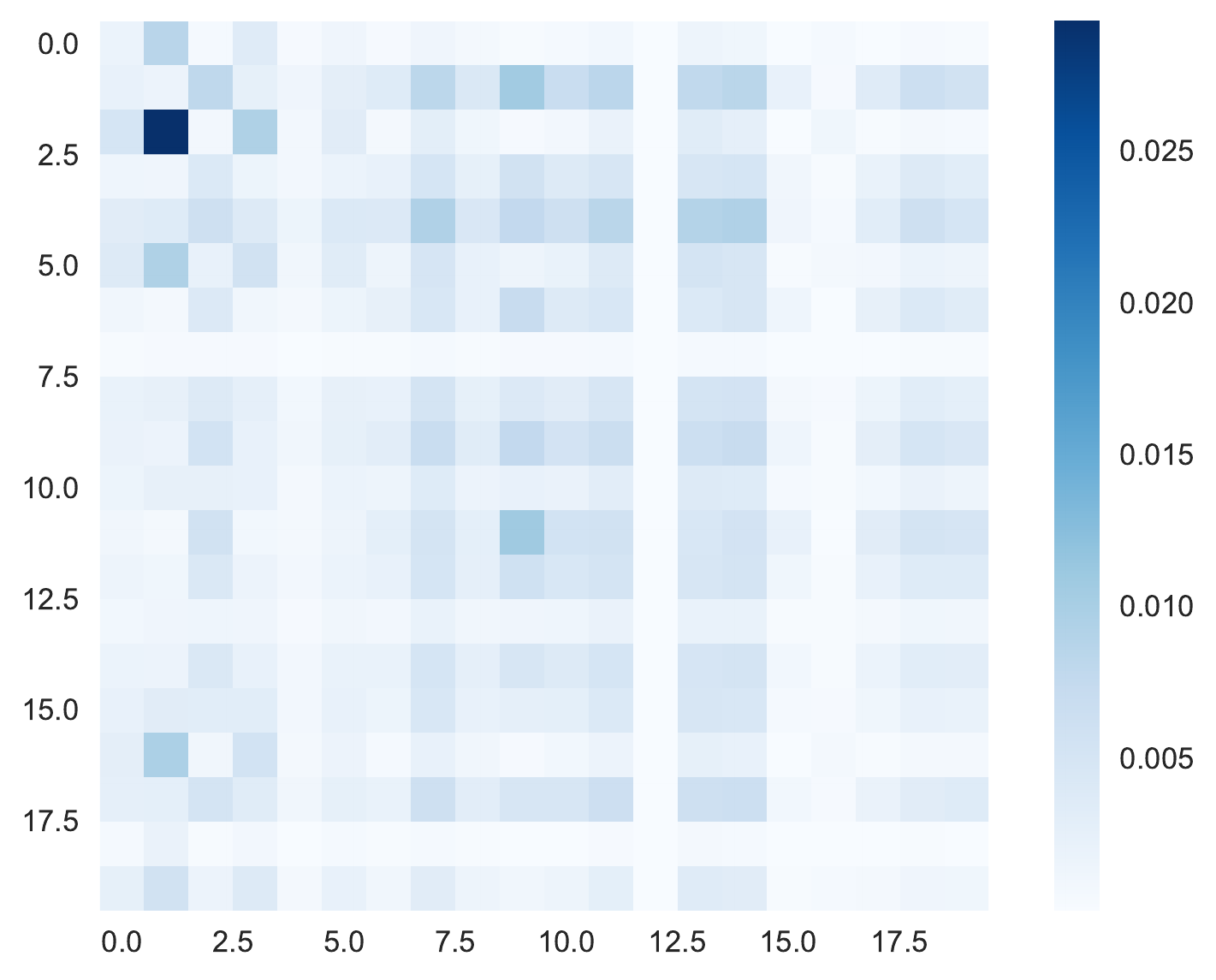}
}%
\subfigure[$\hat{\pi}$ ]{%
\label{fig:sample_pi}%
\includegraphics[width=0.25\columnwidth]{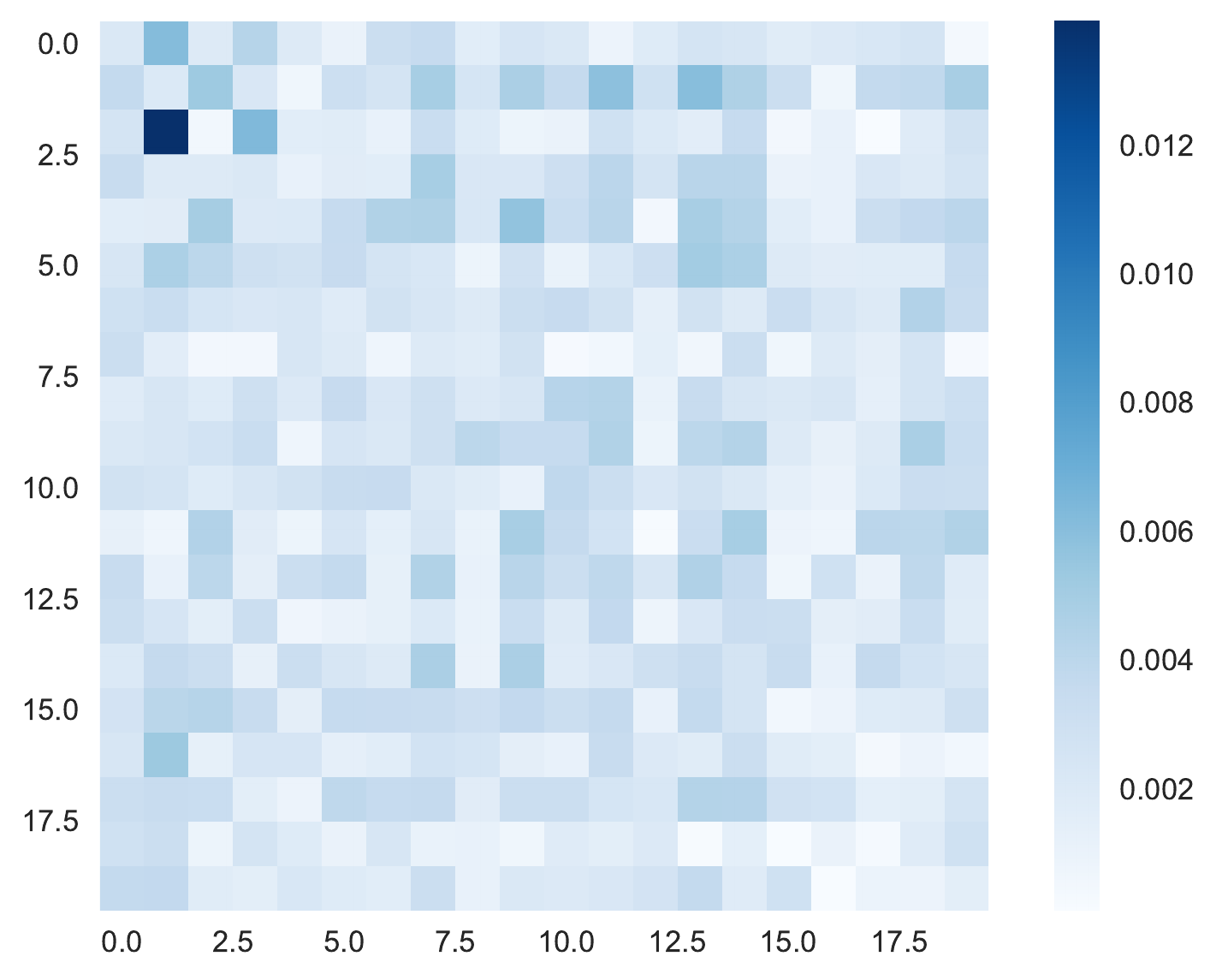}
}%
\subfigure[$\pi_{\text{RIOT}}$]{%
\label{fig:relax_pi}%
\includegraphics[width=0.25\columnwidth]{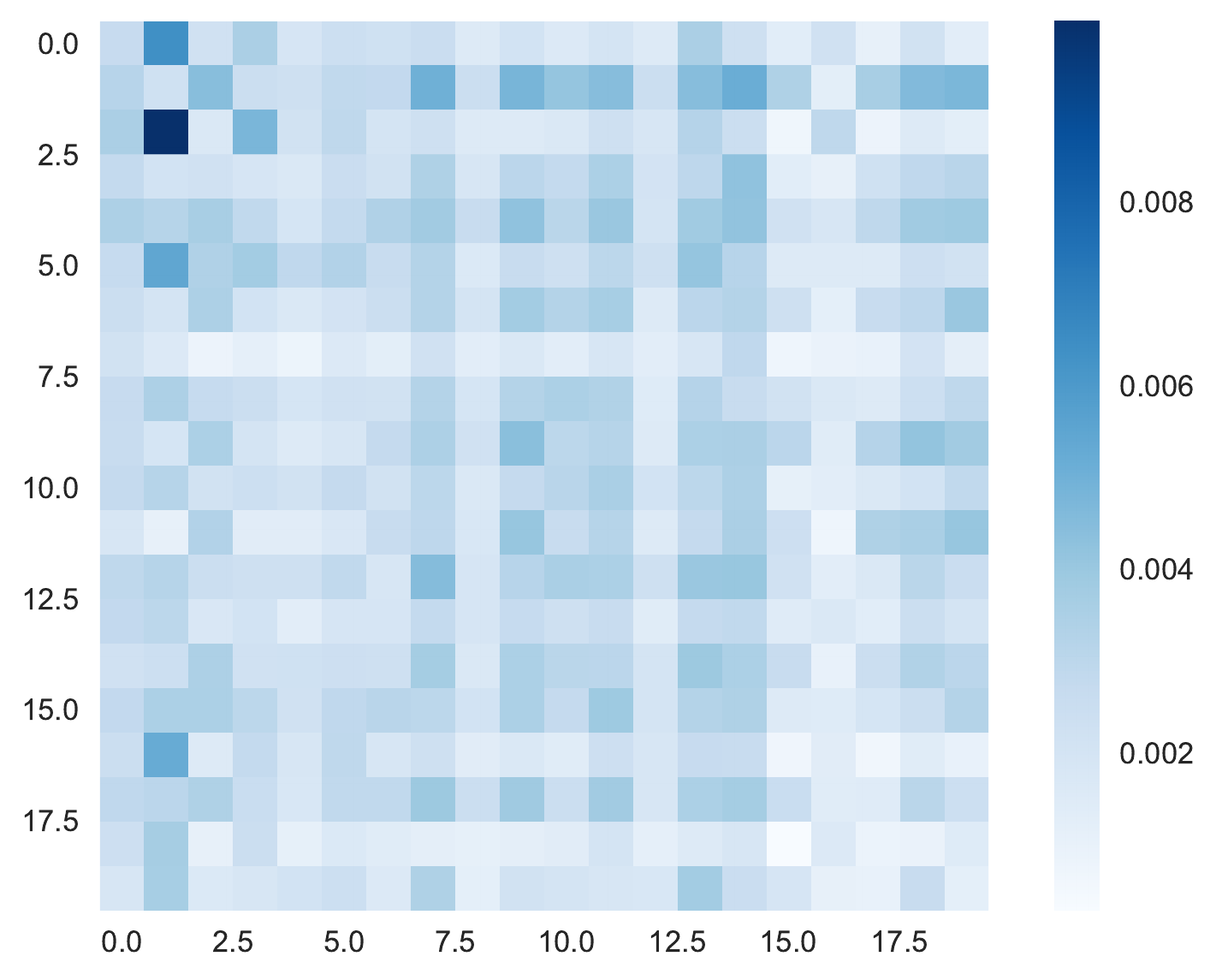}
}%
\subfigure[$\pi_{\text{IOT}}$]{%
\label{fig:fix_pi}%
\includegraphics[width=0.25\columnwidth]{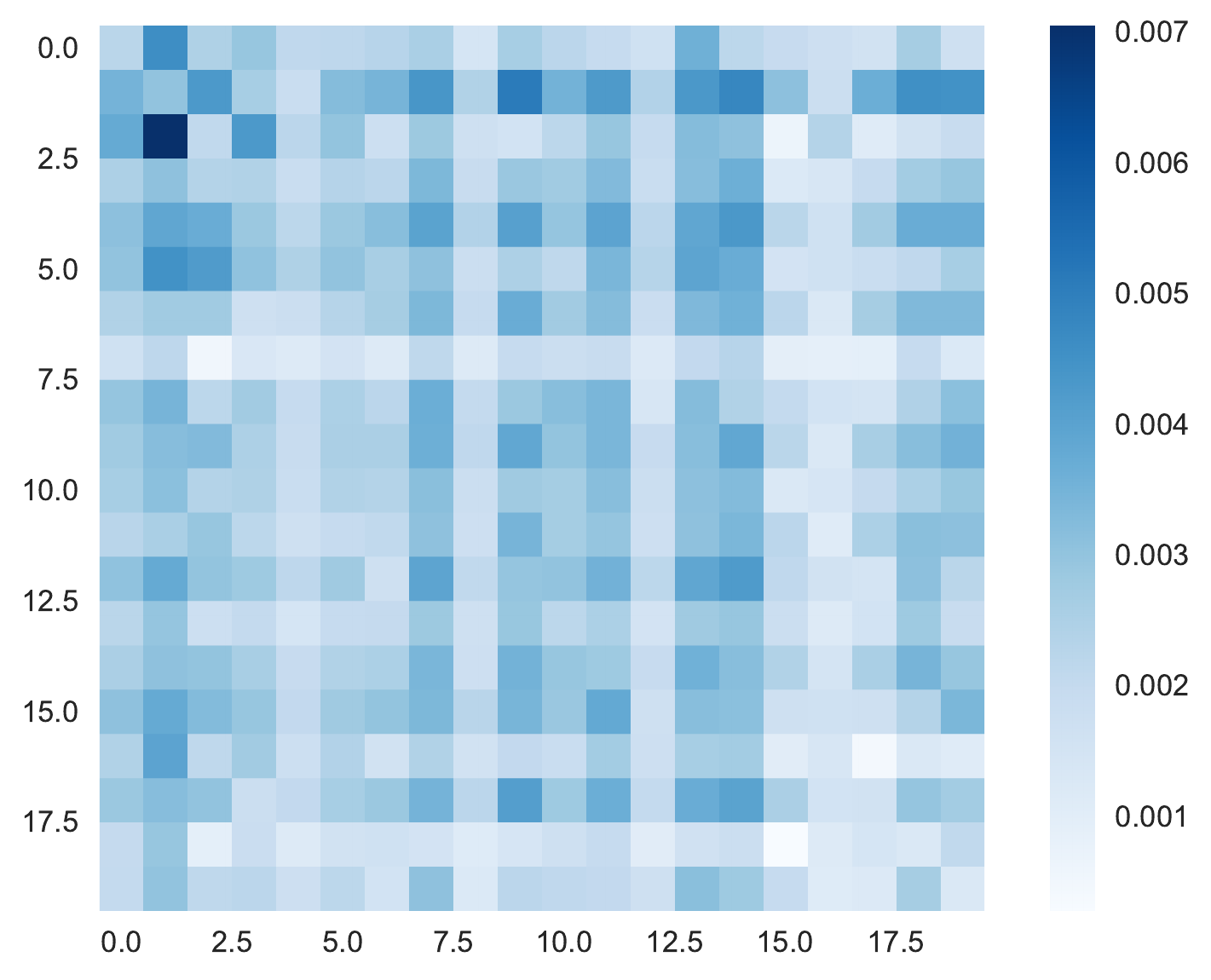}
}%
\caption{Comparison of (a) actual matching matrix $\pi_0$ (b) noised matching matrix $\hat{\pi}$ ($\text{KL}(\pi_0\|\hat{\pi}) = 0.1854$) (c) matching matrix $\pi_{\text{relax}}$ learned by marginal relaxation approach ($\text{KL}(\pi_0\|\pi_{\text{RIOT}}) = 0.1365$) (d) matching matrix $\pi_{\text{fix}}$ learned with fixed marginal approach ($\text{KL}(\pi_0\|\pi_{\text{IOT}}) = 0.1476$) \, (noise size $\sigma=8\times 10^{-3}$)}\label{fig:compare2}
\end{figure}

\subsubsection{Improved Robustness}
To produce Figure \ref{fig:compare1}, we set the number of iterations in outer loop $L=50$, learning rate $s=10$. For each $\sigma \in 10^{-4}\times\{ 1, 5, 10, 50, 10, 500, 1000, 5000\}$ we run algorithm \ref{alg:main} for $L=50$ iterations and record Kullback-Leibler divergence between learned matching matrix $\pi_{\text{IOT}}$, $\pi_{\text{RIOT}}$ and ground truth matching matrix $\pi_0$. Figure \ref{fig:compare1} shows that RIOT with different relaxation parameters demonstrate improved robustness than IOT does when noise size is large. If $\delta$ is set too large ($\delta = 0.05$), however, the entropy term tends to dominate and negatively affect recovery performance when noise size is modestly small. If $\delta$ is tuned carefully ($\delta= 0.01, 0.001$), RIOT can achieve comparable performance even when noise size is quite small. Moreover, we observe that curves corresponding to different $\delta$ intersect with the curve of fixed marginal at different noise size. Therefore, when prior knowledge or statistical estimate of noise size is available, we may tune $\delta$ accordingly to achieve best practical performance.

To produce Figure \ref{fig:compare2}, set noise size $\sigma=8\times10^{-3}$ and relaxation parameter $\delta=0.01$ with other parameters same as those for producing Figure \ref{fig:compare1}. Figure \ref{fig:compare2} visually illustrates $\pi_0$, $\hat{\pi}$, $\pi_{\text{RIOT}}$ and $\pi_{\text{IOT}}$ and we see that noise of this size, significantly corrupts the ground truth matching matrix. $\pi_{\text{RIOT}}$ exhibits less distortion compared to $\pi_{\text{IOT}}$, which demonstrates improved robustness again. Numerical results also back up our observation.
\[\text{KL}(\pi_0\|\hat{\pi}) = 0.1854, \, \text{KL}(\pi_0\|\pi_{\text{RIOT}}) = 0.1365, \, \text{KL}(\pi_0\|\pi_{\text{IOT}}) = 0.1476\]
Compared to IOT, marginal relaxation via regularized Wasserstein distance does help improve the robustness of solution.

\subsubsection{Superior Learning Performance} 
To produce Figure \ref{fig:compare3}, set noise size $\sigma=0.08$, relaxation parameter $\delta=0.001$,  the number of iterations in outer loop $L=100$ and learning rate $s=1$, we then run algorithm \ref{alg:main} to compare the performance of learning cost matrix $C_0$. To avoid non-uniqueness/non-identifiability issue, we use
\[ 
d(C_1, C_2) = \min_{D = \bs{a}\bs{1}^T + \bs{1}\bs{b}T + C_1} \| D - C_2)\|_F
\]
to measure the closeness of cost matrices $C_1$ and $C_2$ and denote the minimizer of $d(C, C_0)$ by $\tilde{C}$. The results are shown below,
\[
d(C_{RIOT}, C_0) = 4.7075, \quad d(C_{IOT}, C_0) = 8.4623
\]
where $C_{RIOT}, C_{IOT}$ are cost matrices learned by RIOT formulation and IOT formulation respectively. Compared to $C_{IOT}$, $C_{RIOT}$ learned via our proposed method almost halves the distance to ground truth  cost matrix. Figure \ref{fig:compare3} 
also illustrates that our model can learn the structure of cost matrix better than IOT does. Our approach improves the learning performance and is able to reveal the structure of ground truth cost matrix.\\

\noindent To sum up, we show that with appropriately tuned relaxation parameter, RIOT is superior to IOT in terms of both robustness and learning performance.


\begin{figure}[ht]
\centering
\subfigure[$C_0$]{%
\label{fig:actual_cost}%
\includegraphics[width=0.3\columnwidth]{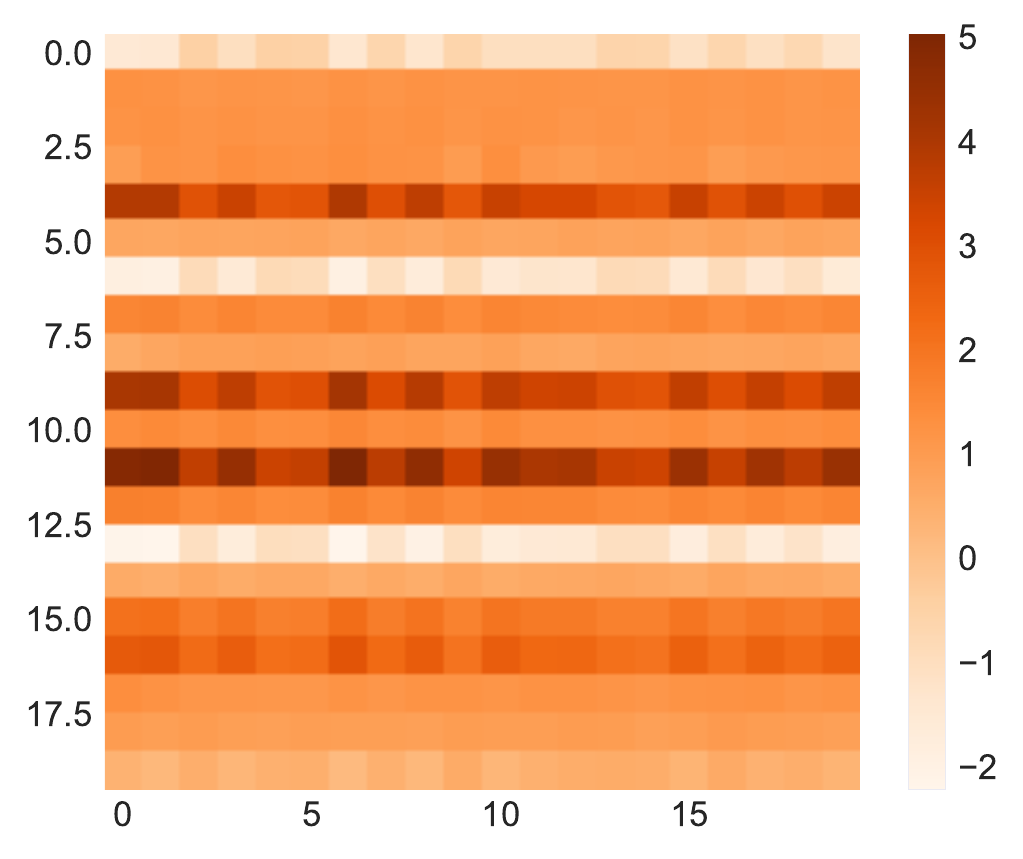}}%
\subfigure[$\tilde{C}_{\text{RIOT}}$]{%
\label{fig:relax_cost}%
\hspace{0.04\columnwidth}
\includegraphics[width=0.3\columnwidth]{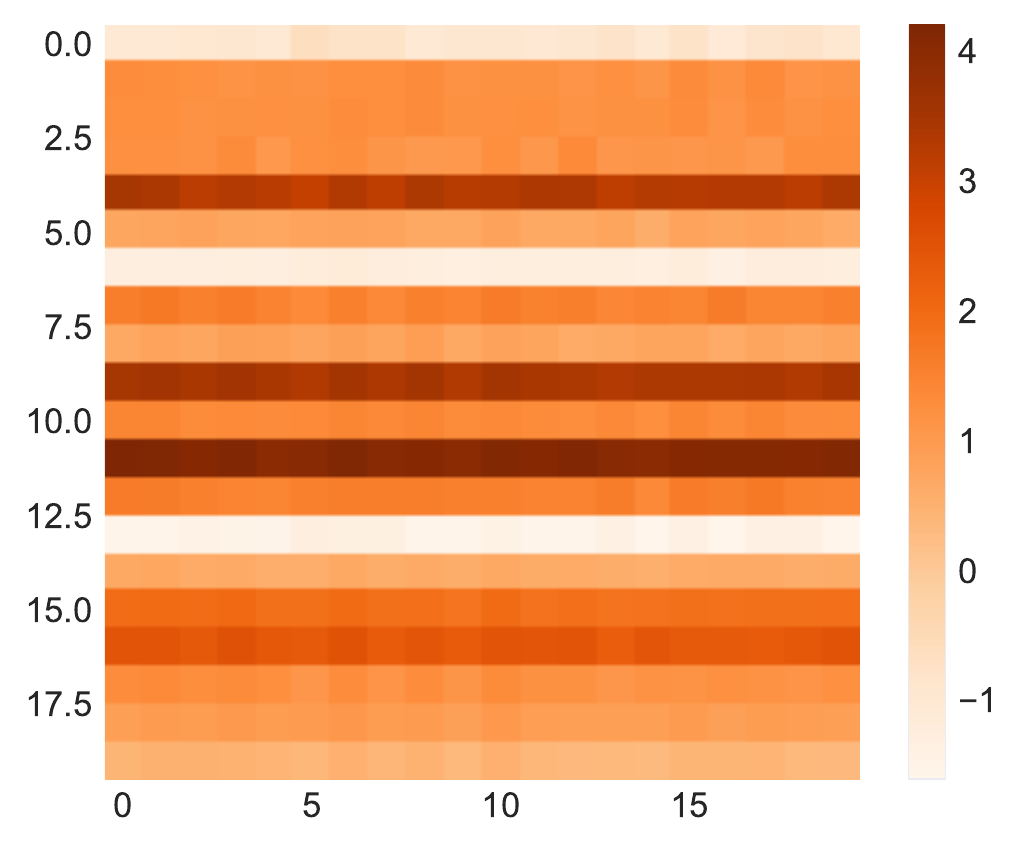}}%
\subfigure[$\tilde{C}_{\text{IOT}}$]{%
\label{fig:fix_cost}%
\hspace{0.04\columnwidth}
\includegraphics[width=0.3\columnwidth]{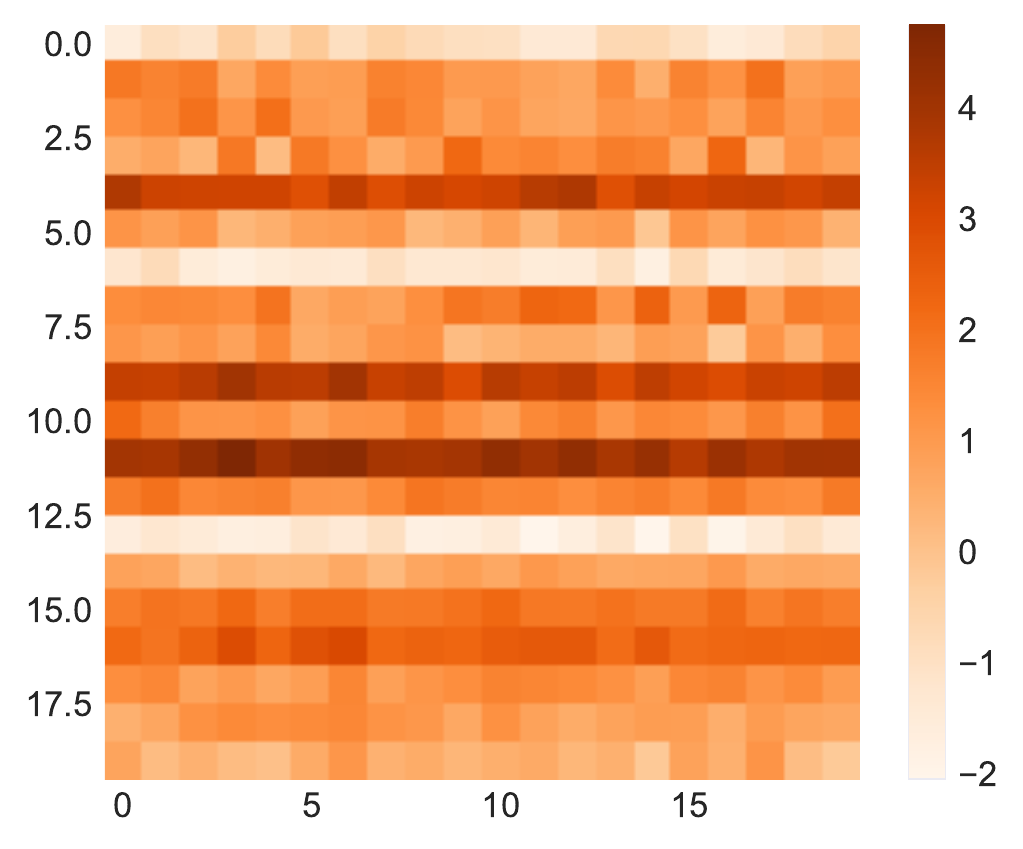}}%
\caption{Comparison of (a) ground truth cost matrix $C_0$ (b) $\tilde{C}_{\text{RIOT}}$, the minimizer of $d(C_{RIOT}, C_0)$ and (c) $\tilde{C}_{\text{IOT}}$, the minimizer of $d(C_{IOT}, C_0)$}\label{fig:compare3}
\end{figure}

\begin{figure}[ht]
\centering
\includegraphics[width=1\columnwidth]{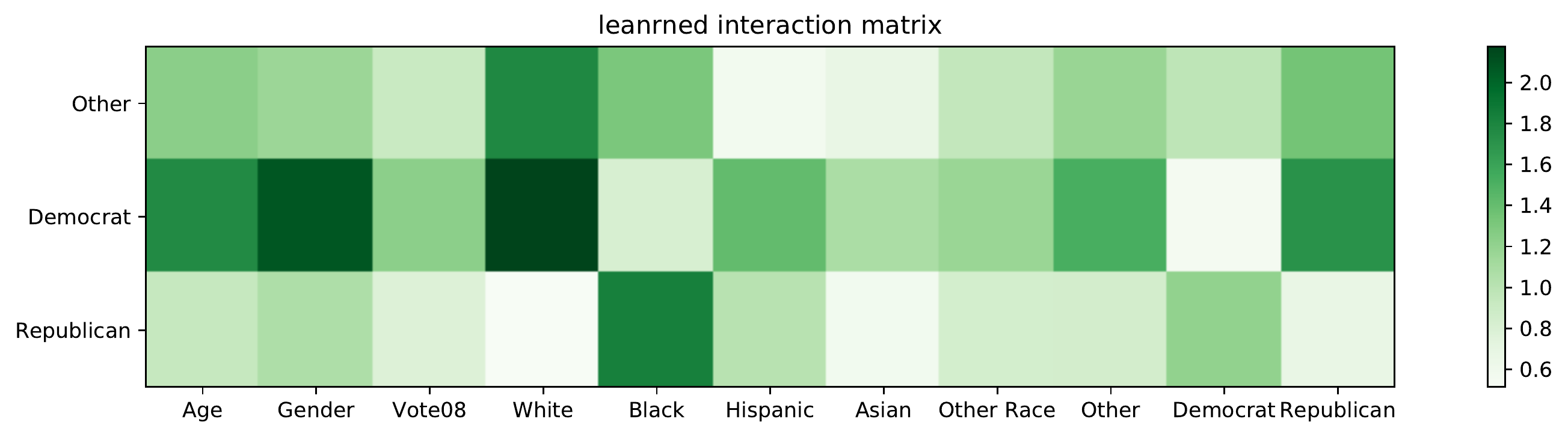}
\caption{Interaction matrix learned by RIOT for election data set}\label{fig:election}
\end{figure}

\subsection{Election data set}
We show in this subsection that RIOT can effectively learn the user-item interaction by applying it to 2012 presidential election data of Florida. The experiment setup is similar to that of \citet{muzellec2017tsallis}\footnote{part of the experiment in this subsection is based on the code kindly shared by Boris Muzellec(https://github.com/BorisMuzellec/TROT)}. 

The data set contains more than $9.23\times 10^{6}$ voters profile, each voter has features like gender, age, race, party and whether the voter voted in 2008 election. 0-1 encoding are used for gender (M:1, F:0) and voting in 2008 (Yes:1, No:0), age is linearly mapped onto $[0,1]$  and we use one-hot encoding for both race and party. We obtain the empirical matching data from exit poll provided by GALLUP\footnote{http://news.gallup.com/poll/160373/democrats-racially-diverse-republicans-mostly-white.aspx}. The empirical matching matrix $\hat{\pi}$ is 3-by-5 matrix, candidates are either Democratic or Republican,  or from a third party and voters are categorized into five races (White, Black, Hispanic, Asian and Other). We use mean profile as features for each race and one-hot encoding as features for candidates. We set $\lambda=1$, for $C_u, C_v$, we randomly generate $m$ and $n$ points from $\mathcal{N}(\bs{0}, 5I_2)$ on plane and use their Euclidean distance matrix as $C_u$ and $C_v$ 
We run RIOT for this data set and the learned interaction matrix is shown in Figure \ref{fig:election}.

In Figure \ref{fig:election}, the lighter the color of a cell is, the lower the cost caused by that feature combination is. Take `Age' column as an example, (`Democratic', `Age') is the darkest and (`Republican', `Age') is the lightest among the column, it means that elder voters are more likely to favor Republican candidate as it has lower cost compared to supporting Democratic candidate. Other cells can be interpreted in a similar manner.

From Figure \ref{fig:election}, we see that most white voters tend to support Republican candidate Romney while black voters tend to support Democratic candidate Obama, Democratic and Republican voters tend to support candidate from their own party, elder voters tend to support Romney while female voters tend to support Obama. All above observations are consistent with CNN's\footnote{http://www.cnn.com/election/2012/results/state/FL/president/} exit polls. This demonstrates that RIOT can learn meaningful interaction preference from empirical matching effectively.

\subsection{New York Taxi data set}
We demonstrate in this subsection that the proposed RIOT framework is able to predict fairly accurate matching on New York Taxi data set \footnote{https://www.kaggle.com/c/nyc-taxi-trip-duration/data}. This data set contains 1458644 taxi trip records from January to June in 2016 in New York city. Each trip record is associated with one of the two data vendors (Creative Mobile Technologies, LLC and VeriFone Inc.) and contains detailed trip information such as pickup/drop-off time, longitude, latitude and so on. As no unique identifiers of taxis are provided, we can not predict new matching on individual level. Instead, we predict matching between data vendors and passengers (a passenger is matched with one of the data vendors if a taxi associated with that data vendor rides with the passenger). 

To reflect the proximity of passengers and taxis, we cluster all trip records into 50 regions and plot them in Figure \ref{fig:cluster}. If a passenger and a taxi are in the same region, it indicates they are close to each other and it is desirable to match them up. Further, since we do not have real-time location of taxis, we use the last known drop-off location as taxis' current location. This assumption is usually not true in large time scale as taxis are likely to leave the region and search for next passenger. To alleviate this issue, we only use trip records within a short time period, 6:00-6:30pm on Friday, June 3rd to predict matching of 6:00-6:30pm on Friday, June 10. Moreover, this is typically the rush hour in New York city and location of taxis are not likely to change dramatically during the period. Vendors' features are the distribution of associated taxis across 50 regions, i.e., $U \in \mathbb{R}^{50 \times 2}$, passengers' features are simply the one-hot encoding of their current location, i.e, $V \in \mathbb{R}^{50 \times 50}$. So the interaction matrix $A \in \mathbb{R}^{50 \times 50}$. We set $r = 1$, $\lambda=1$ and randomly generate $m$ and $n$ points from $\mathcal{N}(\bs{0}, 5I_2)$ on plane and use their Euclidean distance matrix as $C_u$ and $C_v$.

The comparison of the actual matching $\pi_{new}$ and the predicted matching $\pi_{predicted}$ is shown in Figure \ref{fig:taxi}. Visually speaking, we see that the predicted matching is able to capture the pattern of actual empirical matching and the prediction is fairly accurate. Quantitative result is also reported, measured in Kullback-Leibler divergence 
\[
\text{KL}(\pi_{new}||\pi_{predicted}) = 0.1659.
\]

\begin{figure}[ht]
\centering
\includegraphics[width=1.1\columnwidth]{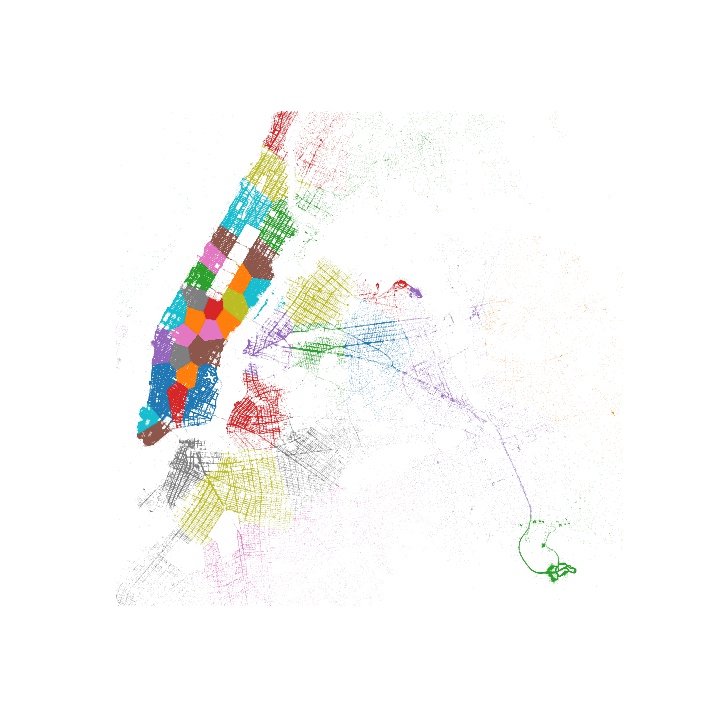}
\caption{Clusters of Taxi Trip Records}\label{fig:cluster}
\end{figure}

\begin{figure}[ht]
\centering
\includegraphics[width=1\columnwidth]{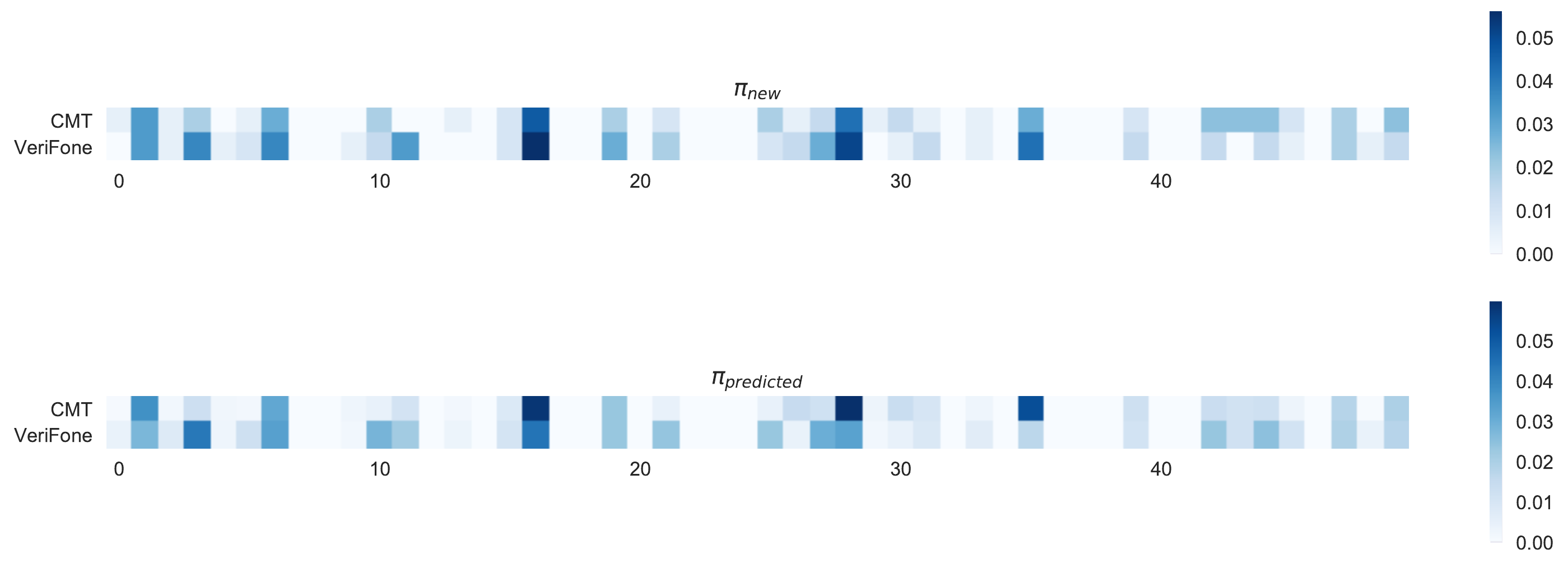}
\caption{Comparison between actual matching $\pi_{new}$ (top) and predicted matching $\pi_{predicted}$ (bottom) between 6:00-6:30pm on Friday, June 10th in New York city. Ticks of $x$-axis are labels of regions.}\label{fig:taxi}
\end{figure}

\subsection{Marriage data set}
In this subsection, we illustrate the applicability of our model in suggesting new matching and it can make more accurate and realistic recommendations than conventional recommender systems do. Once the interaction between two sides of matching market is learned, one may use that to predict matching for new groups and make recommendations accordingly. We compare our RIOT with baseline random predictor model (Random), classical SVD model \citep{koren2009matrix} and item-based collaborative filtering model (itemKNN) \citep{cremonesi2010performance}, probabilistic matrix factorization model (PMF) \citep{mnih2008probabilistic} and the state-of-art factorization machine model (FM) \citep{rendle2012factorization}. To fit conventional recommender systems in our setting, one possible approach is simply treating each cell of matching matrix as rating and ignoring the underlying matching mechanism. In RIOT, we set $\lambda=\lambda_u=\lambda_v=1$, relaxation parameter $\delta=0.001$ and use polynomial kernel $k(\bs{x},\bs{y})=(0.2\bs{x}^T\bs{y} + 0.8)^2$. 

We evaluate all models on Dutch Household Survey (DHS) data set \footnote{https://www.dhsdata.nl/site/users/login} from 2005 to 2014 excluding 2008 (due to data field inconsistency). After data cleaning, the data set consists of 2475 pairs of couple. For each person we extract 11 features including education level, height, weight, health 
and 6 characteristic traits, namely irresponsible, accurate, ever-ready, disciplined, ordered, clumsy and detail-oriented. Education levels are first categorized into elementary, middle and high and then mapped linearly onto $[0,1]$. Height and weight are normalized by dividing the largest height/weight. Health and characteristic features are measured on 0-5 scale and rescaled onto $[0,1]$. We use k-means algorithm to cluster men/women into $n_{\text{cluster}}=50$ groups, respectively. We select each cluster center as representative. Performing clustering can better illustrate the applicability of our model in distributional setting and also helps reduce problem size. We train all models on training data set and measure error between predicted and test matching matrix by root mean square error (RMSE) and mean absolute error (MAE) using 5-fold cross-validation. The result is shown in Table \ref{table:error}.

\begin{table}[h]
\centering
\begin{tabular}{|c|c|c|c|c|c|c|c|c|}
\hline
     & Random  & PMF     & SVD  & itemKNN  & RIOT               & FM      \\ \hline
RMSE & 54.5 & 8.4 & 29.9  & 2.4 &  \textbf{2.3} & 3.6 \\ \hline
MAE  & 36.6 & 2.0 & 16.8  & 1.6 &  \textbf{1.5}          & 2.8 \\ \hline
\end{tabular}
\caption{Average error of 5-fold cross-validation measured in RMSE and MAE $(\times10^{-4})$}\label{table:error}
\end{table}


In both measures,  RIOT beats other conventional RS competitors. The comparison clearly shows that being able to take supply limitation into consideration and capture matching mechanism is of critical importance in suggest matching in such context and our proposed RIOT model can do a better job than conventional recommender systems do.

\section{Conclusion} \label{conclusion}
In this paper, we develop a novel, unified, data-driven inverse-optimal-transport-based matching framework RIOT which can learn adaptive, nonlinear interaction preference from noisy/incomplete empirical matching matrix in various matching contexts. The proposed RIOT is shown to be more robust than the state of the art IOT formulation and exhibits better performance in learning cost. Moreover, our framework can be extended to make recommendations based on predicted matching and outperforms conventional recommender systems in matching context.

In the future, our work can be continued in multiple ways. First, our model does batch prediction for a group of users and items and we would like to develop online algorithm to deal with streaming data and make matching suggestion for a previous unseen user/item in an online fashion. A recent method proposed by \citet{perrot2016mapping} that allows to update the plan using out-of-sample data without recomputing might be useful. From business standpoint, we may study optimal pricing within our framework, i.e., how to set a reasonable price and adjust item distribution in a most profitable way \citep{azaria2013movie}. In addition, we hope to combine impressive expressiveness of deep neural networks to further boost the performance of our proposed model.







\newpage
\appendix
\renewcommand{\thesection}{\Alph{section}}
\section{Proof of Lemma \ref{lemma:1}} \label{app:1}
\begin{proof}
Denote $U^\star(\bs{\mu}, \bs{\nu}) = \{\pi| \pi\bs{1} = \bs{\mu}, \pi^T\bs{1} = \bs{\nu}\}$. It is easily seen that $U(\bs{\mu}, \bs{\nu}) \subset U^\star(\bs{\mu}, \bs{\nu})$ and 
\[
    \min_{\substack{\pi_1 \in U(\bs{\mu}_1, \bs{\nu}_1)\\ \pi_2 \in U(\bs{\mu}_2, \bs{\nu}_2)}} \frac{1}{2}\|\pi_1 - \pi_2\|_F^2 \ge \min_{\substack{\pi_1 \in U^\star(\bs{\mu}_1, \bs{\nu}_1)\\ \pi_2 \in U^\star(\bs{\mu}_2, \bs{\nu}_2)}} \frac{1}{2}\|\pi_1 - \pi_2\|_F^2
\]
Consider the Lagrangian function of right hand side minimization problem 
\begin{equation*}
\begin{split}
L(\pi_1, \pi_2, \bs{\lambda}_1, \bs{\mu}_1, \bs{\lambda}_2, \bs{\mu}_2) &= \frac{1}{2}\|\pi_1 - \pi_2\|_F^2 \\
&- \bs{\lambda}_1^T(\pi_1\bs{1} - \bs{\mu}_1) - \bs{\mu}_1^T(\pi_1^T\bs{1} - \bs{\nu}_1)\\
&- \bs{\lambda}_2^T(\pi_2\bs{1} - \bs{\mu}_2) - \bs{\mu}_2^T(\pi_2^T\bs{1} - \bs{\nu}_2)
\end{split}
\end{equation*}
The KKT condition is
\[
\begin{cases}
    \frac{\partial L}{\partial \pi_1} = (\pi_1 - \pi_2) - \bs{\lambda}_1 \bs{1}^T - \bs{1} \bs{\mu}_1^T = 0 \\
    \frac{\partial L}{\partial \pi_2} = (\pi_2 - \pi_1) - \bs{\lambda}_2 \bs{1}^T - \bs{1} \bs{\mu}_2^T = 0 \\
    \frac{\partial L}{\partial \bs{\lambda}_1} = \pi_1\bs{1} - \bs{\mu}_1 = 0 \\
    \frac{\partial L}{\partial \bs{\mu}_1} = \pi_1^T\bs{1} - \bs{\nu}_1 = 0 \\
    \frac{\partial L}{\partial \bs{\lambda}_2} = \pi_2\bs{1} - \bs{\mu}_2 = 0 \\
    \frac{\partial L}{\partial \bs{\mu}_2} = \pi_2^T\bs{1} - \bs{\nu}_2 = 0 \\
\end{cases}
\]
By solving KKT condition, we have $\bs{\lambda}_1 = \frac{\Delta \bs{\mu}}{n} + x \bs{1}$ and $\bs{\mu}_1 = \frac{\Delta \bs{\nu}}{m} + y \bs{1}$ where $x = \frac{\bs{\lambda}_1^T\bs{1}}{m}$ and $y = \frac{\bs{\mu}_1^T\bs{1}}{n}$ and 
\begin{align*}
    \|\pi_1 - \pi_2\|_F^2 &= \|\bs{\lambda}_1 \bs{1}^T + \bs{1}\bs{\mu}_1^T\|_F^2 \\
    &= \|\frac{1}{n} \Delta \bs{\mu} \bs{1}^T + \frac{1}{m} \bs{1} \Delta \bs{\nu}^T + z\bs{1}\bs{1}^T\|_F^2 \qquad (z = x + y) \\
    &=  \sum_{i=1}^m \sum_{j=1}^{n} (z + \frac{1}{m} \Delta \bs{\nu}_{j} + \frac{1}{n} \Delta \bs{\mu}_{i})^2 \\
    &= \ (mn z^2 + \frac{1}{m}\|\Delta \bs{\nu}\|_2^2 + \frac{1}{n} \|\Delta \bs{\mu}\|_2^2) \\
    &\ge \frac{m\|\Delta \bs{\mu}\|^2_2 + n \|\Delta \bs{\nu}\|_2^2}{mn}
\end{align*}
\end{proof}

\newpage
\section{Proof of Lemma \ref{lemma:2}} \label{app:2}
\begin{proof}
\[ 
f(\bs{a}, \bs{b}) = \bs{x}^T A \bs{x}  - 2\bs{f}^T \bs{x} + \|M\|^2_F
\]
where $\bs{x} =  \begin{bmatrix}\bs{a} \\ \bs{b}\end{bmatrix}$, $\bs{f} = [(M \bs{1})^T, \bs{1}^TM]^T$  and $A = \begin{bmatrix} 
n I_{m\times m} & \bs{1}_m \bs{1}_n^T \\
\bs{1}_n \bs{1}_m^T & mI_{n \times n}
\end{bmatrix}$.
Note that $A$ is a positive semi-definite matrix, the algebraic multiplicity of its $0$ eigenvalue is $1$ and $\mbox{null}(A) = \mbox{span}\{ [\bs{1}_m^T, -\bs{1}_n^T]^T\}$. Moreover, $\bs{f} \perp \mbox{null}(A)$, hence the quadratic form $f(\bs{a}, \bs{b})$ admits a minimum in $\mbox{null}(A)^\perp$ and it is straightforward to obtain 

\begin{align*}
\min_{\bs{a}, \bs{b}} f(\bs{a}, \bs{b}) &= \min_{\bs{x}}\bs{x}^T A \bs{x}  - 2\bs{f}^T \bs{x} + \|M\|^2_F \\
&= \|M\|^2_F - \bs{f}^T A^+ \bs{f}
\end{align*}
where $A^+$ is the Moore-Penrose inverse of matrix $A$. 
If $M$ can not be written as $M = \bs{a}\bs{1}^T + \bs{1}\bs{b}^T$, then the minimum of $f(\bs{a}, \bs{b})$ must be positive, hence
\[ 
f(\bs{a}, \bs{b}) \ge \min_{\bs{a},\bs{b}}f(\bs{a}, \bs{b}) = \|M\|^2_F - \bs{f}^T A^+ \bs{f} > 0
\].



\end{proof}

\newpage
\section{Extension of RIOT model to learn $C_u$ and $C_v$ jointly}\label{app:extension}
In this section, we extend proposed RIOT model to settings where $C_u$ and $C_v$ are unknown and need to be learned jointly with the main cost matrix $C(A)$. 
Following the same derivation, we end up with an optimization problem almost identical to the one in equation \eqref{eq:primal}, i.e.,

\begin{equation}\label{eq:extension_primal}
\min_{A, \bs{\mu}\in \Sigma_m, \bs{\nu}\in\Sigma_n, C_u, C_v} -\sum_{i=1}^m \sum_{j=1}^n \hat{\pi}_{ij}\log\pi_{ij}+  \delta \big(d_{\lambda_u}(C_u, \bs{\mu}, \hat{\bs{\mu}}) + d_{\lambda_v}(C_v, \bs{\nu}, \hat{\bs{\nu}})\big)
\end{equation}
except for the fact that now we need to optimize two additional variables $C_u$ and $C_v$. Genericly, inverse problems are usually not well-posed, in our cases, if no constraints imposed on $C_u, C_v$, one could trivially let, say $C_u = 0_{m\times m}, C_v = 0_{n \times n}$. To avoid such ill-posedness, we assume that $C_u \in \mathcal{M}^m \cap \Sigma^{m\times m}, C_v \in \mathcal{M}^n \cap \Sigma^{n \times n}$, where $\mathcal{M}^d$ is the cone of $d\times d$ distance matrix  and $\Sigma^d$ is $d-1$ simplex (see section \ref{background} for definition). Other regularization can also be explored.

By strong duality, we may convert equation \eqref{eq:extension_primal} to its dual problem in a similar fashion as equation \eqref{eq:dual},

\begin{equation*}
\min_{A,\bs{\mu},\bs{\nu}, C_u, C_v}\max_{\bs{z},\bs{w}} -\sum_{i=1}^m \sum_{j=1}^n \hat{\pi}_{ij}\log\pi_{ij} + \delta \big( \langle \bs{z},\bs{\mu}\rangle + \langle \bs{z}^{C_u},\hat{\bs{\mu}}\rangle
+\langle \bs{w},\bs{\nu}\rangle + \langle \bs{w}^{C_v},\hat{\bs{\nu}}\rangle \big)
\end{equation*}
where $z^{C_u}_j = \frac{1}{\lambda_u}\log \hat{r}_j-\frac{1}{\lambda_u} \log(\sum_{i=1}^m e^{\lambda_u(z_i - {C_u}_{ij})})$ and $w^{C_v}_j = \frac{1}{\lambda_v}\log \hat{c}_j-\frac{1}{\lambda_v} \log(\sum_{i=1}^n e^{\lambda_v(w_i - {C_v}_{ij})})$. 

One way to solve eqution \eqref{eq:extension_dual}, without too many changes of proposed algorithm in section \ref{derivation}, is to rewrite it as 

\begin{equation}\label{eq:extension_dual}
\min_{A,\bs{\mu},\bs{\nu}} \min_{C_u, C_v}\max_{\bs{z},\bs{w}} -\sum_{i=1}^m \sum_{j=1}^n \hat{\pi}_{ij}\log\pi_{ij} + \delta \big( \langle \bs{z},\bs{\mu}\rangle + \langle \bs{z}^{C_u},\hat{\bs{\mu}}\rangle
+\langle \bs{w},\bs{\nu}\rangle + \langle \bs{w}^{C_v},\hat{\bs{\nu}}\rangle \big)
\end{equation}
and alternatively update three groups of variables $(A, \bs{\mu}, \bs{\nu})$, $(C_u, C_v)$ and $(\bs{z}, \bs{w})$.

\subsection{Update $(A, \bs{\mu}, \bs{\nu})$, with $(C_u, C_v)$ and $(\bs{z}, \bs{w})$ fixed}
Once $(C_u, C_v)$ and $(\bs{z}, \bs{w})$ fixed, $\bs{z}^{C_u}, \bs{w}^{C_v}$ are fixed as well, hence the optimization problem at this stage becomes
\begin{equation*}
\min_{A,\bs{\mu},\bs{\nu}} -\sum_{i=1}^m \sum_{j=1}^n \hat{\pi}_{ij}\log\pi_{ij} + \delta \big( \langle \bs{z},\bs{\mu}\rangle + \langle \bs{w},\bs{\nu}\rangle \big)
\end{equation*}
where constants are omitted. This minimization problem is identical to that in subsection \ref{subsection:update1}. Please see detailed update scheme for $(A, \bs{\mu}, \bs{\nu})$ there.

\subsection{Update $(C_u, C_v)$, with $(A, \bs{\mu}, \bs{\nu})$ and $(\bs{z}, \bs{w})$ fixed}
If $(A, \bs{\mu}, \bs{\nu})$ and $(\bs{z}, \bs{w})$  are fixed, $\pi$ is also fixed as it is the regularized OT plan determined by parameters $(A, \bs{\mu}, \bs{\nu})$, hence the optimization problem at this stage becomes
\begin{equation*}
\min_{C_u \in \mathcal{M}^m \cap \Sigma^{m \times m}, C_v \in \mathcal{M}^n \cap \Sigma^{n \times n}}  \delta \big(\langle \bs{z}^{C_u},\hat{\bs{\mu}}\rangle + \langle \bs{w}^{C_v},\hat{\bs{\nu}}\rangle \big)
\end{equation*}
and can be further splitted into two independent optimization problems
\begin{equation}\label{eq:extension_seperate}
\min_{C_u \in \mathcal{M}^m \cap \Sigma^{m \times m}}  \delta \langle \bs{z}^{C_u},\hat{\bs{\mu}}\rangle, \qquad \min_{C_v \in \mathcal{M}^n \cap \Sigma^{n \times n}}  \delta \langle \bs{w}^{C_v},\hat{\bs{\nu}}\rangle 
\end{equation}
which can be solved simultaneously.

Both $\mathcal{M}^d$ and  $\Sigma^{d \times d}$ are convex sets, so is their intersection. Therefore we can perform projected gradient method to solve two seperate minimization problems in equation \eqref{eq:extension_seperate}.
\subsection{Update $(\bs{z}, \bs{w})$, with $(A, \bs{\mu}, \bs{\nu})$ and $(C_u, C_v)$ fixed}
When $(A,\bs{\mu}, \bs{\nu})$ and $(C_u, C_v)$ are fixed, $\pi$ is also fixed, we then only need to solve 
\[
\max_{\bs{z},\bs{w}}\langle \bs{z},\pi\bs{1}\rangle + \langle \bs{z}^{C_u},\hat{\bs{\mu}}\rangle
+\langle \bs{w},\pi^T\bs{1}\rangle + \langle \bs{w}^{C_v},\hat{\bs{\nu}}\rangle
\]
and one immediately recognizes this is equivalent to applying Sinkhorn-Knopp algorithm to compute
$d_{\lambda_u}(C_u,\bs{\mu}, \hat{\bs{\mu}})$ and $d_{\lambda_v}(C_v,\bs{\nu}, \hat{\bs{\nu}})$.\\

To summarize, to jointly learn $C(A)$, $C_u$ and $C_v$, we formulate an optimization problem similar to that in equation \eqref{eq:primal} and propose an alternating algorithm to solve it by alternately update $(A, \bs{\mu}, \bs{\nu})$, $(C_u, C_v)$ and $(\bs{z}, \bs{w})$. Practically, the update of $(C_u, C_v)$ requires expensive projection onto $\mathcal{M}^d \cap \Sigma^{d \times d}$, therefore we suggest learning $C_u, C_v$ first and then using RIOT formulation to learn the main cost matrix $C(A)$, rather than learning three cost matrices simultaneously.

\newpage
\bibliography{Learn-to-Match.bib}

\end{document}